\documentclass[runningheads]{llncs}
\usepackage{lineno,hyperref}
\usepackage{array,multirow}
\usepackage{graphicx}
\usepackage{subfig}
\usepackage{mathrsfs}  
\usepackage{amsfonts}
\usepackage{amsmath}
\usepackage{lmodern}
\usepackage{mathtools}
\usepackage{blkarray, bigstrut}

\newcommand\xdownarrow[1][2ex]{%
	\mathrel{\rotatebox{90}{$\xleftarrow{\rule{#1}{0pt}}$}}
}

\makeatletter
\renewcommand\@biblabel[1]{}
\makeatother
\modulolinenumbers[5]

\usepackage{algorithm2e}

\begin{document}

\title{Decoding Causality by Fictitious VAR Modeling\thanks{The views expressed herein are those of the author and should not be attributed to the IMF, its Executive Board, or its management.}}

\author{Xingwei Hu} 
\authorrunning{X Hu}
\institute{International Monetary Fund\\ Washington, DC 20431, USA \\
	\email{xhu@imf.org}}

\maketitle              

\begin{abstract}
In modeling multivariate time series for either forecast or policy analysis, it would be beneficial to have figured out the cause-effect relations within the data.
Regression analysis, however, is generally for correlation relation and does not accumulate causal effects; also very few researches have focused on variance analysis for causality discovery. 
We first set up an equilibrium for the cause-effect relations using a fictitious vector autoregressive model.
In the equilibrium, long-run relations are identified from noise, and spurious ones are negligibly close to zero.
The solution, called causality distribution, measures the relative strength causing the movement of all series or specific affected ones.
If a group of exogenous data affects the others but not vice versa, then, in theory, the causality distribution for other variables is necessarily zero. 
The hypothesis test of zero causality is the rule to decide a variable is endogenous or not.
Our new approach has high accuracy in identifying the true cause-effect relations among the data in the simulation studies.
We also apply the approach to estimating the causal factors' contribution to climate change, which accrues direct and indirect effects over a long time.

\end{abstract}

\section{Introduction}\label{sect:introduction}

Causal discovery or identification has been one of the few most fundamental goals in machine learning and economics (e.g., Eichler, 2013).
It is preliminary for causal inference.
In the discovery process, in general, an automatic algorithm iteratively identifies the latent causations among the data variables.
However, the latency could hide a deep and complicated data structure over which we merely observe a superficially precedent connection between any two variables. 
The pairwise precedence relations indeed constitutes a directional network with various strength and uncertainty for each edge. 
The bilateral relations, however, are not in a logical and consistent order.
This paper applies a network equilibrium (Hu and Shapley, 2003) and defines a new type of causality.
The solution integrates direct and indirect influences among the data, accrues the real causal relations, and filters out the noisy ones.

Defining causality is a profound convoluted question with many possible answers which do not satisfy everyone (Granger, 1980).
There are dozens of approaches in the literature. 
They generally capitalize on the conditional probability under interventions or the significance of regression coefficients (e.g., Granger, 1969 and 1980; Pearl, 1995).
However, conditional probability involves a parametric probability distribution and the precision in estimating them from the data.
In a regression, the results depend entirely on the appropriate selection of variables, and causal factors that are not incorporated into the regression model cannot be represented in the output.
Adding or dropping a covariate, for example, could make a significant effect insignificant; it could also make an insignificant one significant.
Moreover, the estimated causal relations are not transitive. For example, if $x$ significantly Granger causes $y$ and $y$ significantly Granger causes $z$; however, $x$ may not Granger cause $z$.
In this causal chain, $y$ could merely be a pass-through.
Direct but artificial ones often marginalize indirect causal factors.
In studying climate change, for example, solar activities have had little trending vibration in the past hundred years and seem to have no impact on the up-trending global warming.
The aggregate impact over a thousand years, however, may be no longer negligible. 
Also,  indirect impacts passing through other determinants to the change have played an essential role in the last hundred years.
A regression model, however, only averages the direct effect in the estimate sample, and does not automatically add up the small effects.

In this paper, we propose a new approach that attempts to resolve the above issues. 
Our research is based on two presumptions. 
First, a causal effect could act on sequential occasions, which could occur over a long time and happen multiple times. 
Thus, an aggregate and accrual effect would be suitable for measurement.
Secondly, indirect causal reason passes through one or more intermediaries and gains causal strength through the spillovers from these intermediaries.
The device we use is a sequence of fictitious vector autoregressive (VAR) models in which the causal relations are under investigation.
The shocks to the VAR are natural direct interventions; responses to the shocks are plausible effects.
The limit variance decomposition aggregates both direct and indirect, both short-run and long-run effects.
The decomposition quantifies the bilateral impacts among all variables.
It defines a network with directed edges over which the bilateral impacts are the flow volumes.
Then we define an equilibrium over the network, which resolves potential conflicts in the bilateral influences or plausible causations.
In the literature, Leamer (1985) suggested that variance decomposition may contain causal interpretation and stated that Sims (1980) had already implied that.

The equilibrium has a global and a local version. 
On a global scale, it produces a unique solution when the network is highly connected.
The vector solution quantifies the causality distribution for the movement of the time series variables. 
When there are both exogenous and endogenous variables, the values in the solution for any endogenous one would be statistically insignificant. 
A hypothesis test can formally conduct this.
Similarly, we also test the disconnection among the variables after element-wisely removing these endogenous variables from the fictitious VAR model.
The local version of causality distribution attributes each determinant a fair share of contribution to an effect variable.
The precision of these hypotheses tests is evidenced in our simulation studies.
We also apply the distribution to study climate changes and find which factors are more important than others.

The rest of the paper is organized as follows: 
Section \ref{sect:pairwise} introduces a bilateral interaction network among a vector of time series data by a vector autoregressive model. 
Section \ref{sect:Equilibrium} derives the valuation formula, which measures the amount of causality each time series contains; the solution is the global causality distribution which applies to the whole vector. 
Section \ref{sect:causality_index} studies a causality distribution that measures causal variables' contributions to any effect variable.
Section \ref{sect:identification_algorthms} provides algorithms to identify the causal relations.
Section \ref{sect:simulation_studies} uses the simulation studies to test the effectiveness of the new approach and conducts an empirical study for climate change.
Finally, Section \ref{sect:conclusion} concludes with further comments.
Our exposition is self-contained, and the proofs are in the Appendix.

\section{Pairwise Influences Within Multiple Time Series}\label{sect:pairwise}
 
We study a group of covariance stationary time series,  $y_1$, $y_2$, ..., and $y_n$ (i.e., their means, variances, and covariances remain unchanged over time),
which can be adequately described by a linear model with lagged values of themselves.
We index the time series by $N=\{1,2,\cdots, n\}$.
The data generating process, however, is a black box; we attempt to identify which are endogenous, which are exogenous, and which cause which variables.
For convenience, we stack the time series in a vector $y=\left(\begin{array}{c} y_1 \\ \vdots \\ y_n \end{array} \right )$
and its value at time $t$ is $y_{.t} = \left(\begin{array}{c} y_{1t} \\ \vdots \\ y_{nt} \end{array} \right)$.

\subsection{Basic Bilateral Relations of Interaction}
The bilateral interventions are directional, and all together make a directed network.
Figure \ref{fig:interactive_network} is an example in which ``$\to$" is for one-way directional influence, and ``$\leftrightarrow$" is for two-way directional influence in other figures.
In the network, nodes are the variables, and the directional edges are the intervention relations.
We study a few basic relations which may have similarities to the concepts in the theory of discrete-time Markov chains.
For example, an exogeneity class corresponds to an irreducible class while the endogeneity class corresponds to the transient class.
Familiarity with the Markov chain theory is not required but helpful.

\begin{figure}[ht]
\centering
\parbox{5cm}{
\centering
\includegraphics[height=3.2cm, width=4.5cm]{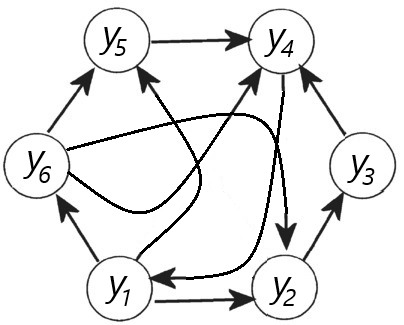}
\caption{An interactive network.}\label{fig:interactive_network}}
\qquad
\begin{minipage}{5cm}
\centering
\includegraphics[height=3.2cm, width=4.5cm]{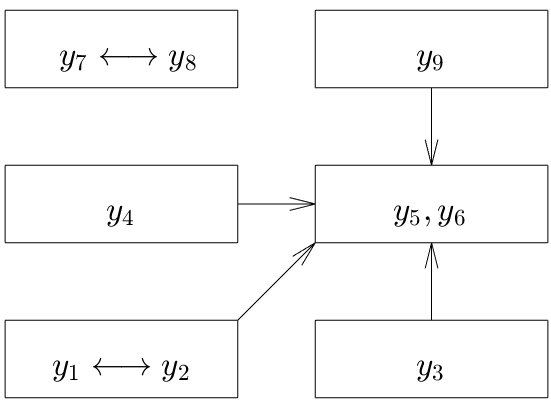}
\caption{Exogeneity classes.}\label{fig:classification}
\end{minipage}
\end{figure}

\subsubsection{Exogeneity Classes:}
There are a few levels of exogeneity (e.g., Engle et al., 1983).
We use the popular version: a class of variables is exogenous to all others if their conditional probability is invariant to the changes of the other time series in the list $y_{.t}$.
Otherwise, they are called endogenous to the other variables.
We focus on the disjoint exogeneity classes such that each variable in a class is endogenous to the other variables in the same class but exogenous to variables outside of the class.
The vector $y_{.t}$ could contain more than one independent exogeneity class and up to one dependent class of endogenous variables.
The exogeneity classes are exogenous to each other and also to the endogenous one.
The casual relation is a class property in that if one variable in a class has a causal effect on another variable, then all variables in the class have the same causal effect.
Figure \ref{fig:classification} illustrates this classification of nine time series variables.
There are five exogeneity classes, four of which affect the endogeneity class $\{y_5, y_6\}$.
The isolated exogeneity class $\{y_7, y_8\}$ evolutes on its own dynamics.
If we model $y_{.t}$ by an unrestricted vector autoregression, however, the coefficients from other variables could still be significant in the equations of $y_7$ and $y_8$ due to the sampling errors.
Conversely, $y_7$ and $y_8$ may also be significant in other equations in the VAR model.
Thus, the plausible causal linkage may not be reliable from the estimated coefficients alone.

\subsubsection{Hierarchy:} In a hierarchy, one causal factor may not directly influence effect ones; also, all direct causations have a one-way direction.
In Figure \ref{fig:hierarchy}, there are three levels of hierarchy. $\{ y_1, y_2\}$ is the exogeneity class but has no direct causal effect to the third level. 
After removing  $\{ y_1, y_2\}$, both $\{y_3, y_4\}$ and $\{y_5\}$ are \textit{de facto} exogeneity classes.
In a VAR model, restrictions should be placed such that some coefficients are zeros to ensure one-way directions are in place.
Also, through a sequence of VARs, conditional exogeneity could be finally figured out.

\begin{figure}[ht]
\centering
\parbox{5cm}{
\centering
\includegraphics[height=3.2cm, width=4.5cm]{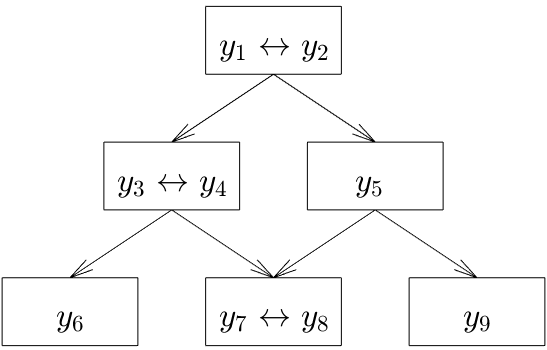}
\caption{Hierarchy.}\label{fig:hierarchy}}
\qquad
\begin{minipage}{5cm}
\centering
\includegraphics[height=3.2cm, width=4.5cm]{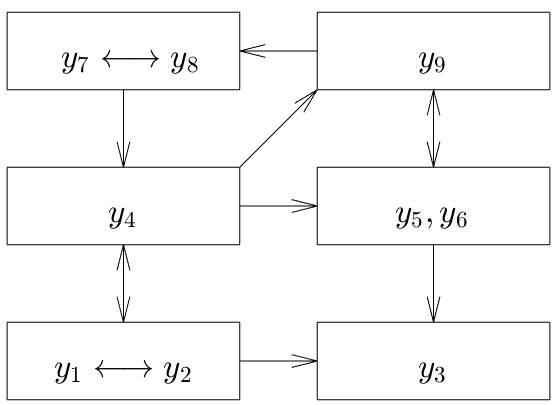}
\caption{Circular.}\label{fig:circular}
\end{minipage}
\end{figure}

\subsubsection{Circular:} In a circular causal structure, there exist one or more circular causal chains.
In Figure \ref{fig:circular}, there are two circular causal chains: $\{y_9\} \to \{y_7, y_8\} \to \{y_4\}\to \{y_9\}$ and $\{y_9\} \to \{y_7, y_8\} \to \{y_4\}\to \{y_5,y_6\}\to \{y_9\}$.
The series $y_3$ is not in any circular chain and has no causal effect on others.
With or without $y_3$, the causal relations, among others, should remain invariant.
A circular chain always belongs to the same class; conversely, a members of a non-singleton exogeneity class always belongs to some circular chain.
The test we demand retains all members in a causal circular and prevents any outsiders from joining.

\subsubsection{Periodic:} This is a special circular chain, and periodicity is also a class property.
The circular causal chain has a minimal period $d$ if the chain returns to the same variable after traveling multiple $d$ periods.
Figure \ref{fig:periodic} has a period $6$.
The number of intermediaries for a circular chain is the same for any member to recurrent to itself.

\begin{figure}[ht]
\centering
\parbox{5cm}{
\centering
\includegraphics[height=3.2cm, width=4.5cm]{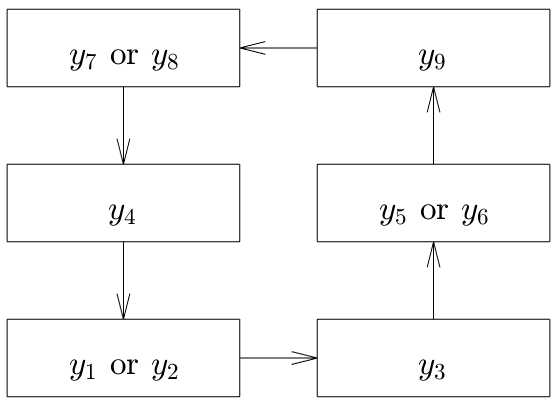}
\caption{Periodic.}\label{fig:periodic}}
\qquad
\begin{minipage}{5cm}
\centering
\includegraphics[height=3.2cm, width=4.5cm]{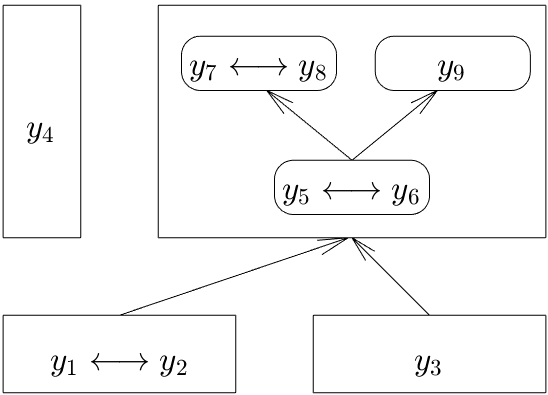}
\caption{Subexogeneity classes.}
\label{fig:subexogeneity}\end{minipage}
\end{figure}

\subsubsection{Sub-exogeneity Classes within the endogeneity class : } The endogeneity class may contain one or more subclasses which are exogenous to the other variables in the endogeneity class.
In Figure \ref{fig:subexogeneity}, $\{ y_5,y_6, y_7,y_8,y_9\}$ is the endogeneity class. Itself has a subclass $\{y_5,y_6\}$ which is exogenous to $\{y_7,y_8,y_9\}$.
To find the contribution of $y_5$ to $y_7$, for example, we put $y_1, y_2$ and $y_3$ as exogenous in the VAR of $y_5,\cdots, y_9$, and ignores $y_4$.

\subsection{The Fictitious VAR Models}

The general unrestricted VAR (GUVAR) with $p$ lags is of the form
\begin{equation}\label{eq:GUVAR}
y_{.t} = \nu+A_{1} y_{.t-1}+\cdots+A_p y_{.t-p}+u_{t}
\end{equation}
where $y_{.t}$, $\nu$, and $u_t$ are $n$ dimensional column vectors, and $A_i$'s are $n \times n$ unknown coefficient matrices.
As usual, we assume that the residual $u_t$ is serially uncorrelated and has a constant variance-covariance $\Sigma$ over time $t$.
However, the model (\ref{eq:GUVAR}) is fictitious and has not differentiated the exogenous and endogenous variables in the list, as well as the causal relations among them. 
It is a mechanism to identify them. 

Restrictions could be placed on the short-run coefficient matrices to reflect the dependence among the component time series.
When there are both exogenous and endogenous variables in the list $y_{.t}$, in theory, the endogenous ones have no short-run impact on the exogenous ones.
Therefore, after some rearrangements, the true VAR(p) has the pattern for all the coefficient matrices $A_1$, $\cdots$, $A_p$:
\begin{equation}\label{eq:pattern_Endo_Exog}
\begin{blockarray}{*{2}{c} l}
\begin{block}{*{2}{>{$\footnotesize}c<{$}} l}
	Exogenous  & Endogenous &  \\
\end{block}
\begin{block}{[*{2}{c}]>{$\footnotesize}l<{$}}
* & 0 \bigstrut[t]& \ Exogenous   \\
* & *             & \ Endogenous  \\
\end{block}
\end{blockarray}
\end{equation}
For two or more isolated classes of exogenous variables, the coefficient matrices of the true VAR(p) have the pattern of block diagonal as
\begin{equation}\label{eq:patern_isolated_classes}
\begin{blockarray}{*{3}{c} l}
\begin{block}{*{3}{>{$\footnotesize}c<{$}} l}
	Class 1  & Class 2 & $\cdots$&  \\
\end{block}
\begin{block}{[*{3}{c}]>{$\footnotesize}l<{$}}
	* & 0 \bigstrut[t]& 0 & \ Class 1 \\
	0 & *             & 0 & \ Class 2 \\
	0 & 0             & * & \ \hspace{.4cm}$\vdots$ \\	
\end{block}
\end{blockarray}
\end{equation}
In practice, however, we are uncertain if there are any exogenous variables. 
Also, an estimated coefficient matrix of (\ref{eq:GUVAR}) would have a nonzero upper right block, due to the sampling errors.

The VAR(p) in (\ref{eq:GUVAR}) can be rephrased to a VAR(1) structure (\ref{eq:VAR1}) by stacking $y_1, y_2, \cdots, y_n$ into a large vector $Y$,
\begin{equation}\label{eq:VAR1}
Y_{t}=V+A Y_{t-1}+U_{t}
\end{equation} 
where
$$
A=\left[
\begin{array}{ccccc}
A_1   &A_2&\ldots&A_{p-1}&A_p   \\ 
I_n   &0  &\ldots&0      &0     \\ 
0     &I_n&      &0      &0     \\ 
\vdots&   &\ddots&\vdots &\vdots\\ 
0     &0  &\ldots&I_n    &0
\end{array}
\right], \
Y=\left[\begin{array}{c}y_1 \\ \vdots \\ y_p\end{array}
\right], \
V=\left[
\begin{array}{c}\nu \\ 0 \\ \vdots \\ 0\end{array}
\right], \
\mathrm{and} \
U_{t}=\left[
\begin{array}{c}u_t \\ 0 \\ \vdots \\ 0
\end{array}
\right].
$$
Besides, $A$ is a $n p \times n p$  matrix, $I_n$ is the $n \times n$ identity matrix, $Y$, $V$ and $U$ are $n p$ dimensional column vectors.
Clearly, $U_t$ is also serially uncorrelated and has variance-covariance $\mathrm{cov}(U_t) = J' \Sigma J$ where $J=\left[\begin{array}{llll}I_n & 0 & \ldots & 0\end{array}\right]$ is a $n \times n p$ matrix.
In general, some elements of a block of an estimated coefficient matrix may be statistically insignificant; but some may be significant.
Even all elements in the block are insignificant, we can still not tell if the block is zero statistically.

For the linear engine (\ref{eq:GUVAR}) of the covariance-stationary $y_{.t}$, the driving forces are the residuals or shocks $u_t$, either positive or negative.
Without the forces, the series remain as a constant vector and no cause-effect relations could be detected.
Though human could also engineer shocks (now called interventions), we ignore the mechanism how they are generated.
With a given $u_t$, the whole series vector responds from time $t$ in which a tiny slice of causal effect would be hardly noticeable.
The aggregated effects through a sequence of $u_t$ could be enormous.
However, a linear regression generally averages, not aggregates, the effects from the innovations.
We analyze the variances of the shock series where the regression amplifies or mitigates these effects.

The shock series is already embedded in the data $y_{.t}$ through the dynamics (\ref{eq:GUVAR}).
Given $Y_t$, the h-step ahead vector from (\ref{eq:VAR1}) is
$$
Y_{t+h} = \sum\limits_{s=0}^{h-1} A^s U_{t+h-s} +  A^h Y_t + \sum\limits_{s=0}^{h-1} A^s V,
$$
and its conditional covariance is
\begin{equation}\label{eq:stacked_variance}
\mathrm{cov}(Y_{t+h}|Y_t) = \sum\limits_{s=0}^{h-1} A^s \mathrm{cov} (U_{t+h-1}) (A^s)' = \sum\limits_{s=0}^{h-1} A^s J' \Sigma J (A^s)'.
\end{equation}
The conditional variance depends on the length of $h$, not on the value of $Y_t$. 
Let $X=\lim\limits_{h\to \infty} \mathrm{cov}(Y_{t+h}|Y_t) = \sum\limits_{s=0}^{\infty} A^s J' \Sigma J (A^s)'$, if it exists. 
Then we have a steady-state variance which satisfies the Lyapunov equation
\begin{equation}\label{eq:Lyapunov}
X = A X A' + J' \Sigma J.
\end{equation}
The solution to the Lyapunov equation solves the following linear equation (cf, e.g., Hamilton, 1994, Equations 10.2.13 and 10.2.18)
$$
\left(I_{n^2p^2} - A \otimes A\right) \operatorname{vec}(X)=\operatorname{vec}(J' \Sigma J)
$$
where $\otimes$ is for the Kronecker product of matrices.
Therefore,
\begin{equation}\label{eq:VAR_limit_variance}
\operatorname{vec}(X) = \left(I_{n^2p^2} - A \otimes A\right)^{-1}  \operatorname{vec}(J' \Sigma J).
\end{equation}
This also implies the covariance of $\operatorname{vec}(X)$ has an even more complicated form.
Let $L$ be a lower triangular matrix obtained by a Cholesky decomposition of $\Sigma$ such that $\Sigma =L L' = \sum\limits_{j=1}^n L_j L_j'$ where $L_j$ be the $j$th column of $L$. 
Then the contribution of $y_j$ to $X$ is
\begin{equation}\label{eq:VAR_limit_variance_yj}
\left(I_{n^2p^2} - A \otimes A\right)^{-1}  \operatorname{vec}(J' L_j L_j' J).
\end{equation}
Element-wise ratios between (\ref{eq:VAR_limit_variance_yj}) and (\ref{eq:VAR_limit_variance}) provide $y_j$'s contribution in percentage to the limit variance of $y_{.t}$.

The covariance matrix $\Sigma$ could also be written as $\Sigma = (L Q) (LQ)'$ for any orthogonal matrix $Q$.
However, $LQ$ has an alignment issue with the order in $y_{.t}$.
A common practice is to order the components of $y_{.t}$ from the most exogenous to the least exogenous (e.g., Sims, 1980).
This way, the shock from the first component series of $y_{.t}$ affects all others contemporaneously and the $L_1$ vector designates the contemporary effects;
the shock from the last component series only affects the series itself contemporaneously where $L_n$ has only one non-zero element.
However, (\ref{eq:VAR_limit_variance_yj}) includes not only the direct contemporaneous impact.

For a finite $h$, we take the first $n$ elements from $Y_{t+h}$ in (\ref{eq:stacked_variance}) to get
$$
\mathrm{cov}(y_{t+h}|y_t) = \sum\limits_{s=0}^{h-1} J A^s J' \Sigma J (A^s)' J' = \sum\limits_{s=0}^{h-1} \Phi_s L L' \Phi_s' = \sum\limits_{s=0}^{h-1} \sum\limits_{j=1}^n \Phi_s L_j L_j' \Phi_s'
$$
where $\Phi_s=J A^s J'$.
Thus, the variance of the h-step forecast of variable $y_i$ is
$$
\mathrm{cov}(y_{i,t+h}|y_t) = e_i' \mathrm{cov}(y_{t+h}|y_t) e_i
=
\sum\limits_{s=0}^{h-1} \sum\limits_{j=1}^n \left( e_i' \Phi_s L_j \right)^{2}
$$
where $e_i$ is the $i$-th column of $I_n$.
Therefore, the amount of forecast variance of variable $y_i$, $\mathrm{cov}(y_{i,t+h}|y_t)$, accounted for by the shocks to variable $y_j$ is given by $\omega_{ij, h}$,
\begin{equation}\label{eq:omega_ijh}
\omega_{ij, h} = \frac{\sum\limits_{s=0}^{h-1} \left(e_i' \Phi_s L_j \right)^{2} }{ \sum\limits_{s=0}^{h-1} \sum\limits_{z=1}^n \left(e_i' \Phi_s L_z \right)^{2} }.
\end{equation}
The asymptotic distribution of (\ref{eq:omega_ijh}) could be derived (e.g., L$\mathrm{\ddot{u}}$tkepohl, 2005, page 111); however, it does not  account for the identity $\sum\limits_{j=1}^n \omega_{ij, h} = 1$ and the non-negativity of (\ref{eq:omega_ijh}). 
Thus, the distribution cannot be used in the usual way for the hypothesis test of $\omega_{ij, h} = 0$.
Theorem \ref{thm:limit_variance_decomposition} shows the existence of $\omega_{ij, h}$ in the infinite horizon of $h$.
We let the unconditional covariance $\omega_{ij} = \lim\limits_{h\to \infty} \omega_{ij, h}$ and denote the $n \times n$ matrix $\Omega = [\omega_{ij}]$, which can be exactly calculated from (\ref{eq:VAR_limit_variance}) and (\ref{eq:VAR_limit_variance_yj}).

\begin{theorem}\label{thm:limit_variance_decomposition} 
If $y_{.t}$ is covariance stationary, then $\lim\limits_{h\to \infty} \omega_{ij, h}$ exists and does not depend on any value of $Y_t$.
\end{theorem}

\begin{proof} 
See Appendix A1. 
\end{proof}

Unlike a causal inference, the discovery of causal relation should be invariant in some way.
It is a YES or NO question.
It does not rely on the unit the variables use or the sign and size of the coefficients in regression.
A causal relation between $y_i$ and $y_j$ does not merely depend on the coefficients and their standard errors in the regression of $y_i$ and $y_j$.
It also relies on the causal structure in $y_{.t}$ and the global uncertainty in the modeling.
The formulas (\ref{eq:VAR_limit_variance}), (\ref{eq:VAR_limit_variance_yj}), and (\ref{eq:omega_ijh}) integrate both the coefficient matrices and the uncertainty.
Meanwhile, (\ref{eq:pattern_Endo_Exog}) and (\ref{eq:patern_isolated_classes}) are two common relational structures in $y_{.t}$.
The following two theorems indicate the invariance of patterns and linearity.

\begin{theorem}\label{thm:pattern_exgo_endo} 
If all the coefficient matrices $A_i$ have the pattern (\ref{eq:pattern_Endo_Exog}), then $\Omega$ has the same pattern.
\end{theorem}

\begin{proof} 
See Appendix A2. 
\end{proof}

\begin{corollary}\label{cr:block_diagonal}
If all $A_i$ have the diagonal pattern (\ref{eq:patern_isolated_classes}), then $\Omega$ is a block lower triangle matrix.
\end{corollary}

\begin{theorem}\label{thm:invariant_linearity} 
$\Omega$ is invariant under any linear transform of any time series from $y_{.t}$.
\end{theorem}

\begin{proof} 
See Appendix A3. 
\end{proof}

The pattern for the coefficient matrices $A_i$ restricts the short-run direct linkages.
If these restrictions carry on in the long run, then the pattern preserves in $\Omega$ as it integrates both short-run and long-run.
The block diagonal pattern (\ref{eq:patern_isolated_classes}) is a special case of (\ref{eq:pattern_Endo_Exog}); thus, Corollary \ref{cr:block_diagonal} follows.
However, $\Omega$ is not necessary a block diagonal; the order of the exogeneity classes matters.
If class $i$ precedes class $j$ in the VAR, then $i$ has a directional impact on $j$ in $\Omega$; but not vice versa.
As an example of the linearity invariance, the emission of carbon dioxide is a causal factor to global warming, measured either in the Celsius scale or the Fahrenheit scale.
The invariance also makes it possible to compare the causality contributions across the variables in $y_{.t}$.

\subsection{Instantaneous Effects by Structural VAR}
When the variables instantaneously affects each other, we could write a general structural VAR (SVAR) form as
\begin{equation}\label{eq:GUSVAR}
y_{.t} = \nu + A_0 y_{.t} + A_1 y_{.t-1} + \cdots + A_p y_{.t-p} + u_t.
\end{equation}
Then
\begin{equation}\label{eq:SVAR2VAR}
y_{.t} = (I_n-A_0)^{-1} \left [\nu + A_1 y_{.t-1} + \cdots + A_p y_{.t-p} + u_t \right ]
\end{equation}
and
\begin{equation}\label{eq:SVAR1}
Y_{t} = (I_n-A_0)^{-1} J V + (I_n-A_0)^{-1} J A Y_{t-1} + (I_n-A_0)^{-1} J U_{t}.
\end{equation} 
Taking variance on both sides of (\ref{eq:SVAR1}) and letting $t\to \infty$, we obtain the Lyapunov equation for $X$,
$$
X = \left [(I_n-A_0)^{-1} J A \right ] X \left [(I_n-A_0)^{-1} J A \right ]' + J' (I_n-A_0)^{-1} \Sigma (I_n-A_0')^{-1} J.
$$
and its solution is
\begin{equation}\label{eq:limit_variance_SVAR}
\operatorname{vec}(X) = \left(I_{n^2p^2} - \tilde A \otimes \tilde A \right)^{-1}  \operatorname{vec}(J' (I_n-A_0)^{-1} \Sigma (I_n-A_0')^{-1} J).
\end{equation}
where $\tilde A = \left [(I_n-A_0)^{-1} J A \right ]$.
The proportion contributed by $y_j$ is
\begin{equation}\label{eq:limit_variance_SVAR_yj}
\left(I_{k^2p^2} - \tilde A \otimes \tilde A \right)^{-1}  \operatorname{vec}(J' (I_n-A_0)^{-1} L_j L_j' (I_n-A_0')^{-1} J).
\end{equation}
The element-wise ratios between (\ref{eq:limit_variance_SVAR_yj}) and (\ref{eq:limit_variance_SVAR}) normalizes the contributions.

\begin{theorem}\label{thm:SVAR_variance_decomposition} 
If $y_{.t}$ is covariance stationary and $||(I_n-A_0)^{-1}||\le 1$, then both (\ref{eq:limit_variance_SVAR_yj}) and (\ref{eq:limit_variance_SVAR}) exist and do not depend on any value of $y_{.t}$.
\end{theorem}

\begin{proof} 
See Appendix A4. 
\end{proof}

\section{A Cause-Effect Equilibrium}\label{sect:Equilibrium}

The bilateral influence matrix $\Omega$ from the general unrestricted VAR (\ref{eq:GUVAR}) is a Markov transition matrix, upon which one may define irreducibility, ergodicity, and periodicity. 
However, each element of the matrix contains a heterogeneous uncertainty and, in an empirical analysis, no element in the matrix would be zero due to the sampling errors.
Of course, many elements would be statistically insignificant when the uncertainty is concerned.
Ignoring the uncertainty, therefore, all variables in $y_{.t}$ are artificially irreducible, ergodic, and aperiodic.

\subsection{Global Causality Distribution}
There are many inconsistencies in $\Omega$ if the transition has a causal interpretation.
The transitivity does not hold in the variance decomposition of $\Omega$. 
Secondly, indirect causal relations are not properly linked. 
Also, spurious causal linkages are not removed.
Lastly, there exists heterogeneous uncertainty in $\Omega$.

To overcome the inconsistencies, we consider the $t$-step influence $\Omega^t$ where $\Omega^t(i,j)$ is the $y_j$'s indirect influence over $y_i$, passing through $t-1$ intermediaries including both $y_i$ and $y_j$. 
The $t$-step matrix make indirect influences direct.
In Figure \ref{fig:hierarchy}, the first level has a direct influence to the third level in $\Omega^2$.
The matrix also integrates any intermediate effects as $\Omega^t = \Omega \Omega^{t-1} = \Omega^2 \Omega^{t-2} =\cdots.$
The Sun's causal effects on global warming, for example, could accumulate for many years and repeatedly through many intermediates (e.g., oceanic evaporation and ozone).
Finally, for an empirical $\Omega$, the indirect influences eventually become stable over time, 
$$
\lim\limits_{t\to\infty} \Omega^t = 1_n (\pi_1, \pi_2, \cdots, \pi_n)
$$
where $1_n$ is the $n \times 1$ vector with all $ones$, $\pi_i\ge 0$, and $\sum\limits_{i=1}^n \pi_i = 1$. 
Therefore, the average direct and indirect influence also have the same limit,
$$
\lim\limits_{z \to \infty} \frac{1}{z}\sum\limits_{t=1}^z \Omega^t = 1_n (\pi_1, \pi_2, \cdots, \pi_k).
$$
In the limits, the $i$th column remains the same across all rows.
It is $y_i$'s causality to the movement of the whole vector $y_{.t}$. 
So we call $\pi_i$ to be $y_i$'s global causality.

The row vector $\pi = (\pi_1, \pi_2, \cdots, \pi_k)$ also satisfies the following counterbalance equilibrium (Hu and Shapley, 2003) over the directional network $\Omega$:
\begin{equation}\label{eq:counterbalance_equilibrium}
\pi = \pi \Omega.
\end{equation}
In the counterbalance, on the one hand, $y_i$'s causality $\pi_i$ to the movements of $y_{.t}$ derives from its effect on all components of $y_{.t}$. 
On the other hand, it also distributes its $\pi_i$ to others which cause $y_i$'s changes.
In the distribution of $\pi_i$, the fair division is through the $i$th row of $\Omega$, which is the percentages of $y_i$'s variability explained by each individual time series.
So $y_j$ gets $\pi_i \omega_{ij}$ from $\pi_i$.
Similarly, $\pi_j$ also distributes $\pi_j$ to others and $y_i$ gets $\pi_j \omega_{ji}$.
Looping on all $j$ in $N$, we get the following equation
\begin{equation}\label{eq:pi_i}
\pi_i = \sum\limits_{j \in N} \pi_j \omega_{ji}
\end{equation}
and (\ref{eq:counterbalance_equilibrium}).
So, the $i$th column of $\Omega$ stipulates the weights with which $y_i$ collects its causality distribution. 

The causality distribution $\pi$ is a mixed cooperative and non-cooperative solution.
It is noncooperative because $\sum\limits_{i \in N} \pi_i = 1$; one series' gain means another's loss.
It is also cooperative because (\ref{eq:pi_i}) implies that a larger $\pi_j$ improves $\pi_i$ more, as long as $\omega_{ji}>0$ remains the same.
Also by  (\ref{eq:pi_i}), increasing $\omega_{ji}$ also enhances $\pi_i$.
A long-run of competitions and cooperations finally reaches the ultimate solution $\pi$.

The distribution $\pi$ could also come from an endogenously weighting.
Ignoring the spillover effects, the $i$th row of $\Omega$ is a proper assignment of responsibility or reason for $y_i$'s movement alone; but it is not a proper measurement for the whole vector of $y_{.t}$.
We assign a number $\pi_i$ to weigh the $i$th row where important variable gets more weight and less important one gets less weight.
Then the weighted assignment of responsibility is $\pi \Omega$, which also measures each variable's contribution to the movement of the vector $y_{.t}$.
So $\pi$ is proportional to $\pi \Omega$.
Setting $\sum\limits_{i=1}^n \pi_i = 1$, we get the equation (\ref{eq:counterbalance_equilibrium}).

Another application of the counterbalance equilibrium can be found in a sorting algorithm for big data (Hu, 2020).
Instead of using $\Omega$ in the equilibrium, one could also apply the frequency decompositions (Geweke 1982; Kaminski et al. 2001) or the generalized variance decomposition (Lanne and Nyberg, 2016).
Both these two decompositions also have a unit sum for each row.
An iterative algorithm for (\ref{eq:counterbalance_equilibrium}) is
\begin{equation}\label{eq:solve_equilibrium}
\pi^{(t)} = \pi^{(t-1)} \Omega
\end{equation}
for a given nonnegative vector $\pi^{(0)}$ such that it has a unit sum.
The limit of $\pi^{(t)}$ would be $\pi$ for an empirical $\Omega$.

\subsection{Noise Filtration by $\pi$}
The general unrestricted VAR in (\ref{eq:GUVAR}) is highly artificial as it ignores the cause-effect relations within the vector $y_{.t}$.
Consequently, this may result in overwhelmingly many insignificant elements in $\Omega$.
One solution is to iteratively identify the causal relations and then re-specify the VAR model using already identified relations.
For example, hypothesis tests could be used to make the VAR from general to specific and from unrestricted to restricted.
By speculating a possible relation after observing some signals, one generally sets up a hypothesis on the relation.
Finally, he or she extracts the evidence from the data to approve or reject the hypothesis.

The distribution $\pi$ mitigates the uncertainty in $\Omega$ by its spillover effects from the third parties.
Indeed, the spillover effects take two further actions: confirmation of real comparative causality and off-setting of noisy ones. 
Consequently, $\pi_i$ could still be insignificant even if many elements of $\Omega$'s $i$th column are statistically significant.
If $y_i$ has a real comparative causality in the directional network, then $\Omega$ has consistent large values in the $i$'s column, especially for variables with large values in $\pi$.
If the advantage is a noise, on the other hand, then the sampling errors could largely account for the positive entires in the $i$th column.

To see the effects of a small variation of $\Omega$ on $\pi$, let us first shock both $\Omega(j,i)$ and $\Omega(k,i)$ by $\delta$.
This can be seen as $y_i$ has some real improvement.
Secondly, we shock $\Omega(j,i)$  by $\delta$ and $\Omega(k,i)$ by $-\delta$.
This can be a case of pure white noise.
By Theorem \ref{thm:dpii_dOmegaji}, the effects add up when two shocks have the same sign while they offset each other when the shocks are opposite.
The small-sized derivative in Theorem \ref{thm:dpii_dOmegaji} also indicates that $\pi_i$'s standard error is small compared to those for the elements in $\Omega$'s $i$th column.
In the theorem, we use three notations for simplicity:
the matrix $Z_i$ is the transpose of $\Omega$ with the $i$th row and the $i$th column removed;
the column vector $\alpha_{ji}$ takes the $j$th row from $\Omega$ and then drops its $i$th element;
and $\pi_{-i}$ is the row vector when $\pi_i$ is removed from $\pi$.

\begin{theorem}\label{thm:dpii_dOmegaji} For any $i,j\in N$,
$$
\frac{\mathrm{d} \pi_i}{\mathrm{d} \omega_{ji}} 
= 
\frac{\pi_j}{1-\omega_{ji}} \
\frac{1_{n-1}' (I_{n-1}-Z_i)^{-1} \alpha_{ji}}{1+1_{n-1}' (I_{n-1}-Z_i)^{-1}\alpha_{ii}} \ge 0
$$
and
$$
\frac{\mathrm{d} \pi_{-i}'}{\mathrm{d} \omega_{ji}} 
=
[I_{n-1}-Z_i]^{-1} \left [ \frac{\mathrm{d} \pi_i}{\mathrm{d} \omega_{ji}} \alpha_{ii} - \frac{\pi_j}{1-\omega_{ji}} \alpha_{ji}\right ].
$$
\end{theorem}
\begin{proof} 
See Appendix A5. 
\end{proof}

The filtration could also be explained by the chains of direct influences. 
The unit sum in each row of $\Omega$ places a zero-sum restriction on the noise of each row. 
So some noise are positive while others are negative.
For a large $t$,
$$
\Omega^t(i,j) = \sum\limits_{i_0=i, i_1\in N, \cdots, i_{t-1}\in N, i_t=j} \omega_{i_0,i_1} \omega_{i_1,i_2} \cdots \omega_{i_{t-1},i_t}
$$
which is a sum of $n^{t-1}$ tiny particles.
Each particle is the product of $t$ elements from $\Omega$ in which some contain positive noise and others contain negative noise.
The summation mitigates the noise when the noise offsets each other.
Adding up the small pieces could make a difference; averaging them, as in a regression, is still small.

\section{Determinants' Responsibilities for an Effect}\label{sect:causality_index}

While mutually causal relations exist in a set of data, the one-way cause-effect linkage is also pervasive.
Say, $y_1$ has a direct influence on $y_2$.
But this does not imply that $y_2$ is necessarily essential to $y_1$'s variability.
Generally, we may put all direct or indirect influencing and influenced data into one huge set and analyze the influences between the classifications
inside the huge set. 
In Figure \ref{fg:transient}, for example, classes 1 and 2 influence class 3; however, the latter has no impact on classes 1 and 2. 
One question is how much is the relative causal importance of classes 1 and 2 to class 3. 
That is, which class is more important to class 3? And by how much?
For example, when applied to climate changes, there are many causal factors accounting for the changes. 
How much is attributed to carbon dioxide emissions? How about the changes in the Sun, the emissions from volcanoes, and the variations in Earth's orbit?

\begin{figure}[ht]
\centering
\parbox{5cm}{
\centering
\begin{tabular}{rcccl}
&               &\fbox{Class 1}      &              & \\
&               &$\xdownarrow[1.2cm]$&              & \\  
$\fbox{Class 2}$&$\longrightarrow$   &\fbox{Class 3}&$\quad$&\fbox{Class 4} 
\end{tabular}
\caption{Causal-transient class.} \label{fg:transient}}
\qquad
\begin{minipage}{5cm}
\centering
$$
\Omega = \left (
\begin{array}{c c c c}
\Omega_1&      &        &   \\
   	    &\ddots&        &   \\
	    &      &\Omega_k&   \\
T_1     &\cdots&T_k     &T_e
\end{array}
\right )
$$
\caption{General pattern for $\Omega$.}\label{fig:patternOmega}
\end{minipage}
\end{figure}

After all variables $y_{.t}$ are identified, say, there are $k$ endogeneity classes $C_1, C_2, \cdots, C_k$, and one endogeneity class $T$. 
Grouping variables by their classes, we have the pattern for the transition matrix in Figure \ref{fig:patternOmega}.
Each $\Omega_s$ is a stochastic matrix, estimated from the VAR with variables from the exogeneity class $C_s$.
Thus the submatrix $\Omega_s$ itself defines a global causal distribution $\pi^{c_s}$ within the exogeneity class $C_s$. 
The $i$th element of the vector $\pi^{c_s}$ quantifies the $i$th variable's responsibility and causality for the dynamics of all variables in the class.
The matrices $T_1, \cdots, T_k$ and $T_e$ are estimated from the VAR with the coefficients matrices pattern of Figure \ref{fig:EstimateTe}.
Therefore, $k+1$ VARs are estimated for Figure \ref{fig:patternOmega}.

\begin{figure}[ht]
\centering
\parbox{5cm}{
\centering
$$
\left (
\begin{array}{ccccc}
*& &      & & \\
 &*&      & &\\
 & &\ddots& & \\
 & &      &*& \\
*&*&\cdots&*&*
\end{array}
\right )
$$
\caption{Restricted coefficients.}\label{fig:EstimateTe}}
\qquad
\begin{minipage}{5cm}
\centering
$$
\Omega = \left  ( 
\begin{array}{cccccc}
\frac{1}{3}&\frac{2}{3}&0          &0          &0          &0\\
\frac{2}{3}&\frac{1}{3}&0          &0          &0          &0\\
0          &0          &\frac{1}{4}&\frac{3}{4}&0          &0\\
0          &0          &\frac{1}{5}&\frac{4}{5}&0          &0\\
\frac{1}{4}&0          &\frac{1}{4}&0          &\frac{1}{4}&\frac{1}{4}\\
0          &\frac{1}{6}&\frac{1}{6}&\frac{1}{3}&\frac{1}{6}&\frac{1}{6}\\
\end{array} \right ).
$$
\caption{An example of $\Omega$.}\label{fig:C1C2T}
\end{minipage}
\end{figure}

\subsection{Causal-transient Members}
For any $y_i$ in $T$, some $y_j$ has causal effect on $y_i$ but not vice versa. 
So $y_i$'s influence on the whole set of times series is transient and temporary, and it dies out gradually.
On the other hand, an exogeneity class accrues causality from the causal-transient $y_i$.
Let $T_e$ be the restriction of $\Omega$ to the endogeneity class $T$ (see Figure \ref{fig:patternOmega}).
Without loss of generality, say, $T$ has $m$ members, namely, $y_{n-m+1}, \cdots, y_n$. 
So $T_e$ is the sub-matrix of the last $m$ columns and the last $m$ rows of $\Omega$. 
As usual, let $1_n$ be the $n\times 1$ vector of ones, and $0_n$ the $n\times 1$ zero vector.
Also, let $I_n$ be the $n\times n$ identity matrix.

\begin{theorem} \label{thm:zero_limit_Te}
$\lim\limits_{t \to \infty} T^t_e = 0_m 0_m'$, the $m \times m$ zero matrix.
\end{theorem}
\proof See Appendix A6.

\begin{corollary} \label{cr:zero_pi_invertible}
$I_m - T_e$ is invertible;
also, if $T\not = \emptyset$, then $y_{.t}$ contains at least one exogeneity class;
and finally, $\pi_i = 0$ for all $i \in T$.
\end{corollary}
\proof See Appendix A7.

By the above theorem, the existence of $T$ does not affect the causality distribution $\pi^{c_s}$.
So we may remove $T$ from $y$ before we calculate the distribution $\pi^{c_s}$,
and the removal reduces the computational cost.
This reduction is extremely important when $y_{.t}$ has a large number of causal-transient members.
For $\Omega$ in Figure \ref{fig:C1C2T}, for example, $T = \{5,6\}, c_1 = \{1,2\}, c_2 = \{3,4\}$ and 
$
T_e = \left ( 
\begin{array}{rl} 
\frac{1}{4}&\frac{1}{4} \\
\frac{1}{6}&\frac{1}{6} \\ 
\end{array} \right ).
$
For $C_1$, the causality distribution is $\pi^{c_1} = (\frac{1}{2},\frac{1}{2})$; and for $C_2$, $\pi^{c_2} = (\frac{4}{19},\frac{15}{19}).$
However, when there are two or more exogeneity classes in $y_{.t}$, the solution to the counterbalance equilibrium (\ref{eq:counterbalance_equilibrium}) is not unique. 
One remedy is to assign some specific causality quota to each class. 

\begin{theorem}\label{thm:pi_with_quota}
If $y_{.t}$ has exogeneity classes $C_1, C_2, \cdots, C_k$ and we assign
$C_s$ with a causality quota $q_s$ where $0 \le q_s \le 1$ and $\sum\limits_{s=1}^k q_s=1$,
then there exists a unique causality distribution to solve the equilibrium
$\pi = \pi \Omega$ and the quota restriction $\sum\limits_{j\in C_s} \pi_j = q_s$ for all $1 \le s \le k$.
\end{theorem}
\proof See Appendix A8.

For the above example of Fig. (\ref{fig:C1C2T}), if we assign the quota $q_1 = .4$ and $q_2 = .6$, then $\pi$ satisfies that
$$
\left \{
\begin{array}{ll}
(\pi_1,\pi_2)
=  
(\pi_1,\pi_2) \left  (\begin{array}{cc} \frac{1}{3} & \frac{2}{3}\\
		\frac{2}{3} & \frac{1}{3} \end{array} \right ),  & \pi_1+ \pi_2 = .4 \\
(\pi_3,\pi_4)
=  
(\pi_3,\pi_4)
\left  (\begin{array}{cc} \frac{1}{4} & \frac{3}{4} \\
	\frac{1}{5} & \frac{4}{5} \end{array} \right ), & \pi_3+\pi_4 = .6 \\
\end{array} \right .
$$
The unique solution is $\pi_1 = \pi_2 =\frac{1}{5}$, $\pi_3=\frac{12}{95}$, and $\pi_4 = \frac{9}{19}$.
They are actually $\left [.4\pi^{c_1}, .6 \pi^{c_2} \right ]$.

\subsection{Local Causality Distribution} 
For any $i \in T$, there may be some $j \in T$ such that $j$ directly influences $i$, i.e., $\omega_{ij}>0$. 
But since $j$ has no global causality in $y_{.t}$, its influence on $i$ must come from members in $N \setminus T$.
In some sense, $j$ acts as messenger only, passing the effects from $N\setminus T$ to $i$. 
This shows that $\omega_{ij}$ is not enough to characterize the causation of $j$ on member $i$. 
To find the causal strength on an element in $T$, let $C$ be an exogeneity class and denote by $\mu_i^c$ the causality by $C$ on the member $i \in T$. 
Conditional on the first move in $\Omega$, $\mu_i^c$ satisfies the following the steady-state equilibrium of inhomogeneous equations
\begin{equation} \label{eq:counterbalance_external}
\mu_i^c = \sum\limits_{j \in T} \omega_{ij} \mu_j^c + \sum_{j \in C} \omega_{ij}.
\end{equation}
Here $\sum\limits_{j \in C} \omega_{ij}$ is the causality absorbed directly by $C$ from $i$ while
the causality absorbed indirectly by $C$ from $i$ is $\sum\limits_{j \in T} \omega_{ij}\mu_j^c$.
Thus, (\ref{eq:counterbalance_external}) is the counterbalance equation with external influences; the players are the elements in $T$, and the external influences come from $C$.

We can approximately approach a solution to (\ref{eq:counterbalance_external}) by the iteration of
\begin{equation}\label{eq:iteraton_counterbalance_external}
\mu^{(t+1)}
= T_e \mu^{(t)} + \left  ( 
\begin{array}{c} 
	\sum\limits_{j \in C} \omega_{n-m+1,j} \\
	\vdots \\
	\sum\limits_{j \in C} \omega_{nj} \\
\end{array}
\right )
\end{equation}
with the starting value $\mu^{(0)} = 0_m$.
Theorem \ref{thm:class_causality2i} establishes the existence and uniqueness of the class's causality on individuals in $T$;
its corollary provides a computational method for the solution.

\begin{theorem} \label{thm:class_causality2i}
There exists a unique solution $\mu_i^c$ to (\ref{eq:counterbalance_external}) for any exogeneity class $C$. 
And the solution is nonnegative.
\end{theorem}
\proof See Appendix A9.

\begin{corollary} \label{cr:limit_class_causality}
For any $i\in T$,
$
\sum\limits_{j \in C} \Omega^t (i,j) = \mu_i^{(t)}
$
in (\ref{eq:iteraton_counterbalance_external}) and
$
\lim\limits_{t\to\infty} \mu^{(t)} = \left ( \begin{array}{c} 
	\mu_{n-m+1}^c \\ \vdots \\ \mu_n^c \\ \end{array} \right ).
$ 
Therefore,
\begin{equation}\label{eq:calculate_mu_c_i}
\left (
\begin{array}{c}
\mu_{n-m+1}^c \\ \vdots \\ \mu_n^c \\
\end{array} 
\right ) 
=  
(I-T_e)^{-1} 
\left  ( 
\begin{array}{c} 
\sum\limits_{j \in C} \omega_{n-m+1,j} \\
\vdots \\ 
\sum\limits_{j \in C} \omega_{nj}
\end{array} 
\right ).
\end{equation}
\end{corollary}
\proof See Appendix A10.

As the transient members has no causality share in $y_{.t}$, the exogneity classes undoubtedly have total ultimate influence over these endogeneity members.
Theorem \ref{thm:total_absorbtion} confirms this.

\begin{theorem} \label{thm:total_absorbtion}
If $C_1, C_2, \cdots, C_k$ are the exogeneity classes in $N$, then
$
\mu_i^{c_1} + \mu_i^{c_2} + \cdots + \mu_i^{c_k}=1
$ 
for any $i \in T$. 
\end{theorem}
\proof See Appendix A11.

\begin{corollary} \label{cr:limit_Omega_T}
$\lim\limits_{t\to\infty} \Omega^t(i,j) = 0$ for $\forall i, j \in T$.
\end{corollary}
\proof See Appendix A12.

To illustrate the class causality on transient members in Figure \ref{fig:C1C2T}, by \ref{eq:calculate_mu_c_i},
$$ 
\left \{
\begin{array}{l}
\left  ( 
\begin{array}{c} \mu_5^{c_1} \\ \mu_6^{c_1}  \end{array} 
\right )
=
[I_2 - T_e]^{-1} 
\left  (
\begin{array}{c} \frac{1}{4}+0 \\ 0+\frac{1}{6}  \end{array} 
\right )
=
\left  (
\begin{array}{c} \frac{3}{7} \\ \frac{2}{7}  \end{array} 
\right ), \\
\left  (
\begin{array}{c} \mu_5^{c_2} \\ \mu_6^{c_2} \\ \end{array} 
\right )
= [I_2 - T_e]^{-1} 
\left (
\begin{array}{c} \frac{1}{4}+0 \\ \frac{1}{6}+\frac{1}{3} \end{array} 
\right )
=
\left  (
\begin{array}{c} \frac{4}{7} \\ \frac{5}{7}  \end{array} 
\right ). \\
\end{array}
 \right .
$$
Thus,
$
\left ( 
\begin{array}{c} \mu_5^{c_1} \\ \mu_6^{c_1}  \end{array} 
\right )
+ 
\left (
\begin{array}{c} \mu_5^{c_2} \\ \mu_6^{c_2} \\ \end{array} 
\right )
= 
\left (
\begin{array}{c} 1 \\ 1 \end{array} 
\right ).
$

As indicated above, $\mu_i^c$ is the total causality of the exogeneity class $C$ over the causal-transient member $i\in T$. 
Now how to re-distribute $\mu_i^c$ among the members inside $C$?
By Theorem \ref{thm:responsibility_index}, this redistribution is through the global causality distribution within the class.
Across all exogeneity classes, the total shares of distribution add up to $1$, claimed in Corollary \ref{cr:sum_casuality_index}.
Therefore, the shares of $[\mu_i^{c_1}\pi^{c_1}, \cdots, \mu_i^{c_k}\pi^{c_k}]$ defines a local causality distribution for $i\in T$.
A large share contributes more to the movement of $y_i$ than a small share, and the distribution varies with $i\in T$.
It reduces to the global causality distribution when there is only one exogeneity class.
This is similar to the responsibility index defined by Shapley (1994), which measures the ultimate controlling power no matter how long the chain of commands is.
In the above example, the local causality distribution is 
$$
\left [\frac{3}{7}\pi^{c_1}, \frac{4}{7}\pi^{c_2} \right ] 
= 
\left [\frac{3}{7}(\frac{1}{2},\frac{1}{2}), \frac{4}{7}(\frac{4}{19},\frac{15}{19}) \right ]
=
\frac{1}{266}\left [57, 57, 32, 120 \right ]
$$ 
for $y_5$ and 
$$
\left [\frac{2}{7}\pi^{c_1}, \frac{5}{7}\pi^{c_2} \right ] 
= 
\left [\frac{2}{7}(\frac{1}{2},\frac{1}{2}), \frac{5}{7}(\frac{4}{19},\frac{15}{19}) \right ]
=
\frac{1}{266}\left [38, 38, 40, 150 \right ]
$$ for $y_6$.
The distributions indicate that $y_4$ accounts twice than $y_1$ for $y_5$'s variability, and about four times for $y_6$'s.
It has different importance on $y_5$ and $y_6$.

\begin{theorem} \label{thm:responsibility_index}
If $\pi_j^c$ is $j$'s causality within the class $C$, then $\lim\limits_{t \to \infty} \Omega^t (i,j) = \pi_j^c \mu_i^c$ for any $i\in T$ and $j\in C$.
\end{theorem}
\proof See Appendix A13.

\begin{corollary} \label{cr:sum_casuality_index} 
$\sum\limits_{s=1}^k \sum\limits_{j \in C_s} \pi_j^{c_s} \mu_i^{c_s} = 1$ for any $i\in T$.
\end{corollary}
\proof See Appendix A14.

\section{Identification Algorithms}\label{sect:identification_algorthms}
Many relations could be set by specifications rather than by estimations.
As we all know, for example, global warming on the surface of the Earth should have no causal effect to the sunspots.
Also, the rooster's crow has no causal effect to the sunrise, even though they are highly correlated.
In this section, we test the causal data structure in $y_{.t}$ by bootstrapping the data.
We order the time series in the VARs using previously calculated $\pi$ values, from the largest to the least.
By (\ref{eq:pi_i}), variable with a large $\pi_i$ is likely to influence more data directly.

\subsection{Endogenous Variables}
When $y_i$ is endogenous, then in theory (\ref{eq:pattern_Endo_Exog}) is the pattern for all coefficient matrices and $\pi_i$ is necessarily zero. 
Thus, the test hypothesis is
$$
H_0 : \pi_i = 0 \quad \mathrm{versus} \quad H_1 : \pi_i>0.
$$
This signals an experiment design which simulates many independently identically distributed (i.i.d.) $\pi_i$'s.
The sample mean $\hat \pi_i$ and its standard deviation $\hat \sigma_i$ help to decide if $\pi_i$ is statistically zero or not; the other statistics, such as sample median and quantiles, are useful in choosing the sample size for bootstrapped datasets.
Under the null hypothesis $H_0$, the z-statistic $\hat \pi_i / \hat \sigma_i$ has an approximate standard normal distribution, by the central limit theorem.
At the $\alpha\%$ significance, therefore, we identify $y_i$ as endogenous if the test statistic $z_i < z_{_{1-\alpha}}$, where $z_{_{1-\alpha}}$ is the critical value.
Otherwise, $y_i$ is classified as exogenous.

\begin{figure}[htb]
\centering
\begin{minipage}{.85\linewidth}
\begin{algorithm}[H]\label{alg:test_endogeneity}
\SetAlgoLined
Pick a large integer $\gamma$ as the number of bootstrapped datasets\;
Take an random ordering of variables for the first VAR\;
$\hat \pi_i \longleftarrow 0$ \ and \ $\hat \sigma_i \longleftarrow 0$\;
\For{$k \gets 1$ to $\gamma$}{
	Bootstrap the original data and estimate the VAR\;
	Calculate $\Omega$ by (\ref{eq:VAR_limit_variance})$-$(\ref{eq:VAR_limit_variance_yj}) and calculate $\pi$\;
	$\hat \pi_i \longleftarrow \hat \pi_i + \pi_i$ \ and \
	$\hat \sigma_i \longleftarrow \hat \sigma_i + \pi_i^2$\;
	Reorder the variables according to the values in $\pi$, from the largest to the least\;
}
$\hat \lambda_i \longleftarrow \hat \lambda_i/\gamma$ \ and \ $\hat \sigma_i \longleftarrow \sqrt{\hat \sigma_i/\gamma}$\;
return $z_i \longleftarrow \hat \lambda_i / \hat \sigma_i $. 
\vskip .3cm
\caption{Hypothesis Test of $\pi_i = 0$ versus $\pi_i > 0$.}
\end{algorithm}
\end{minipage}
\end{figure}

Algorithm \ref{alg:test_endogeneity} details these steps, where $\hat \sigma_i$ is for the estimated standard deviation of $\pi_i$ under $H_0$.
The algorithm actually calculates $z_i$ for all $i=1,2,\cdots,k$. 
In observance of the interdependence among data in $y_{.t}$, we use the algorithm to drop at most one endogenous time series at a time. 
The dropped one has smallest $t$-statistic and its $t$-statistic is less than the cutoff value.
After it is dropped from $y_{.t}$, we run the algorithm again using the updated $y_{.t}$.
We continue this step until no $z$-statistic is less than the cutoff value.

\subsection{Exogeneity Classes}
After removing all endogenous time series from $y_{.t}$, the rest are independent exogeneity classes.
For example, $y_1$ and $y_9$ in Figure \ref{fig:classification} are from disjoint classes.
Without loss of generality, we assume that $y_{.t}$ contains no endogenous variables and
also assume that the estimated global causality $\hat \pi_j > \hat \pi_i$, i.e., $y_j$ precedes $y_i$ in the VAR. 
By Corollary \ref{cr:block_diagonal}, when the coefficient matrices have the pattern (\ref{eq:patern_isolated_classes}), no causal shock would travel from $y_i$ to $y_j$ if they are in disjoint classes.
Given $\pi^{(0)}=e_i'$, therefore, $\pi_j^{(t)} = 0$ for all $t>0$ in (\ref{eq:solve_equilibrium}).
This leads to the following hypothesis:
$$
H_0: \lim\limits_{t\to\infty} \pi_j^{(t)}=0 \ \mathrm{in} \ (\ref{eq:solve_equilibrium})  \ \mathrm{given} \ \pi^{(0)}=e_i'. 
$$

\begin{figure}[htb]
\centering
\begin{minipage}{.85\linewidth}
\begin{algorithm}[H]\label{alg:test_exogeneity}
\SetAlgoLined
Pick a large integer $\gamma$ as the number of bootstrapped datasets\;
Take an random ordering of variables for the first VAR\;
$\hat \pi_j \longleftarrow 0$ \ and \ $\hat \sigma_j \longleftarrow 0$\;
\For{$k \gets 1$ to $\gamma$}{
	Bootstrap the original data and estimate the VAR\;
	Calculate $\Omega$ by (\ref{eq:VAR_limit_variance}) and (\ref{eq:VAR_limit_variance_yj}) \;
	Calculate $\pi$ by (\ref{eq:solve_equilibrium}) given $\pi^{(0)}=e_i'$ \;
	$\hat \pi_j \longleftarrow \hat \pi_j + \pi_j$ \ and \
	$\hat \sigma_j \longleftarrow \hat \sigma_j + \pi_j^2$\;
	Reorder the variables according to the values in $\pi$, from the largest to the least\;
}
$\hat \lambda_j \longleftarrow \hat \lambda_j/\gamma$ \ and \ $\hat \sigma_j \longleftarrow \sqrt{\hat \sigma_j/\gamma}$\;
return $z_j \longleftarrow \hat \lambda_j / \hat \sigma_j $. 
\vskip .3cm
\caption{Testing $(y_i, y_j)$ from different exogeneity classes.}
\end{algorithm}
\end{minipage}
\end{figure}

Algorithm \ref{alg:test_exogeneity} calculates $\pi$ using (\ref{eq:solve_equilibrium}) and the starting value $\pi^{(0)}=e_i'$.
If the returned $z_j$ statistic is less than the cutoff value, we accept the hypothesis $H_0$.

\subsection{From Exogeneity to Endogeneity class}

When an exogenous variable $y_i$ has no causal influence on an endogenous variable $y_j$, such as $y_7$ and $y_6$ in Figure \ref{fig:classification}, 
we also have $\pi_j^{(t)} = 0$ in (\ref{eq:solve_equilibrium}) with $\pi^{(0)}=e_i'$. 
Therefore, Algorithm \ref{alg:test_exogeneity} can test no causality from an exogeneity class to an endogenous variable.

\subsection{Periodicity}
When a class $C$ has a period $d>1$, then for any $i\in C$, $\lim\limits_{t\to\infty} \Omega^t(i,i)$ does not exist.
But $\lim\limits_{t\to\infty} \Omega^{td+1} (i,i) = \cdots = \lim\limits_{t\to\infty} \Omega^{td+d-1} (i,i) = 0$.
Therefore, starting from $\pi^{(0)}=e_i'$, the iteration (\ref{eq:solve_equilibrium}) could be used to test the hypotheses of $d=1$, by slightly modifying Algorithm \ref{alg:test_exogeneity}.

\subsection{Conditional Causation Within the Endogeneity Class}
After finding all exogeneity classes, we further investigate the structure of the endogeneity class by adjusting the VAR model.
The new VAR ignores the exogenous variables which have no causal impact on the endogeneity class.
It uses the other exogeneity classes as the exogenous variables, which have no equations in the VAR.
In Figure \ref{fig:subexogeneity}, for example, the new VAR has $\{ y_5, y_6, y_7, y_8, y_9\}$ as endogenous variables and $\{ y_1, y_2, y_3\}$ as the exogamous variables.
The new VAR is to test if $\{y_5, y_6\}$ is exogenous to $\{ y_7, y_8, y_9\}$ and if $y_9$ and $\{y_7, y_8\}$ mutually affect each other.
Algorithms \ref{alg:test_endogeneity} and \ref{alg:test_exogeneity} could apply to the modified VARs.
This progress could go on and on until all substructures are fully identified.

\subsection{Instantaneous Causality}

Specifically, for example, if $y_j$ affects $y_i$ contemporaneously, then $A_0 = e_i e_j'$ multiplied by an unknown coefficient in (\ref{eq:GUSVAR}).
We compare the VAR (\ref{eq:GUVAR}) and SVAR (\ref{eq:GUSVAR}) models.
If $\pi_i$ has a significant difference, then we say the contemporaneous causal effect is effective.
Let $\tilde \pi_i$ be the global causality distribution for (\ref{eq:GUSVAR}).
Then, the test hypothesis is
$$
H_0 : \tilde \pi_i - \pi_i = 0 \quad \mathrm{versus} \quad H_1 : \tilde \pi_i - \pi_i > 0.
$$
Algorithm \ref{alg:Instantaneous} implements the sample mean and the sample standard error of $\tilde \pi_i - \pi_i$ under $H_0$. 
If the z-score is within the cutoff value, then we accept $H_0$ and claim that the instantaneous causality is ineffective.

\begin{figure}[htb]
\centering
\begin{minipage}{.85\linewidth}
\begin{algorithm}[H]\label{alg:Instantaneous}
\SetAlgoLined
Pick a large integer $\gamma$ as the number of bootstrapped datasets\;
Take an random ordering of variables for the VAR and another ordering for the SVAR\;
$\hat \delta_i \longleftarrow 0$ \ and \ $\hat \sigma_i \longleftarrow 0$\;
\For{$k \gets 1$ to $\gamma$}{
	Bootstrap the original data and estimate the VAR and SVAR\;
	Calculate $\Omega$ by (\ref{eq:VAR_limit_variance}) and (\ref{eq:VAR_limit_variance_yj}) and its associated $\pi$\;
	Calculate $\tilde \Omega$ by (\ref{eq:limit_variance_SVAR}) and (\ref{eq:limit_variance_SVAR_yj}) and its associated $\tilde \pi$\;	
	$\hat \delta_i \longleftarrow \hat \delta_i + (\tilde \pi_i - \pi_i)$ \ and \
	$\hat \sigma_i \longleftarrow \hat \sigma_i + (\tilde \pi_i - \pi_i)^2$\;
	Reorder the variables for the VAR according to the values in $\pi$, from the largest to the least\;
	Reorder the variables for the SVAR according to the values in $\tilde \pi$, from the largest to the least\;
}
$\hat \delta_i \longleftarrow \hat \delta_i/\gamma$ \ and \ $\hat \sigma_i \longleftarrow \sqrt{\hat \sigma_i/\gamma}$\;
return $z_i \longleftarrow \hat \delta_i / \hat \sigma_i $. 
\vskip .3cm
\caption{Testing the significance of $A_0$'s effect on $\pi$.}
\end{algorithm}
\end{minipage}
\end{figure}

\section{Simulation and Empirical Studies} \label{sect:simulation_studies}

This section applies simulation experiments to study the performance of our new estimation methods.
The estimated causal structures are compared with the true ones from which data are generated.
We also list the number of errors the estimation makes.
In terms of exact identification of the actual models, the new approach has a precision of 83\% for the causal structures in Figures \ref{fig:classification} $-$ \ref{fig:subexogeneity}.
For not identified cases, some causal relations are correctly identified, bringing the accuracy to 87.9\%.
On average, each estimated figure also includes one and half spurious causal relations.

This section also conducts empirical research which estimates the determinants' contributions to climate changes. 
We find that the natural factors (the Sun and the Earth) have a majority share.
The emission of carbon dioxide also plays a significant role.

\subsection{Accuracy of the Identification Algorithms}

The estimation could make two types of errors. 
First, it may omit one or more true cause-effect relations.
In Figure \ref{fig:classification}, for example, it may not correctly find the causal influence from $y_3$ to $\{y_5, y_6\}$.
Secondly, it could wrongly add one or more false relations.
In Figure \ref{fig:classification}, there is no relation between $y_4$ and $y_9$. 
But the estimation could wrongly find one for them.
In observance of the class property, it is unnecessary to identify all possible causal relations between any two variables.
In Figure \ref{fig:classification}, for example, if already finding $y_1 \leftrightarrow y_2$ and $y_1\to y_5$, then we do not have to test $y_2 \to y_5$. 
However, misspecification contaminates. 
If one causal relation is mistakenly specified, then any further inferences based on the class property and the misspecification would be vulnerable.
 
For each figure, we simulate 1,000 datasets.
In each dataset, each variable also depends on its lagged value.
The coefficients in all these linear models are randomly generated and ensure the stationarity of the time series.
Each regression has a white noise as residual, and all residuals for the same regression are independently identically distributed.
Residuals across different time series are independently generated.
In Figure \ref{fig:classification}, for example, $y_1$ is a linear function of its own lag and the lag of $y_2$;
$y_5$ is a linear function of the lags of $y_1, y_2, y_3, y_4, y_5, y_6$, and $y_9$.

We use the following options: $100$ sample size for each dataset and $.05$ significance level in the algorithms.
In the identification progress, the fictitious VARs could have different lag lengths if a criterion (BIC or AIC, for example) is used to select the lag length.
For the VARs to be consistent, we use the length that generates the data, which is $2$.
For each simulated dataset, we apply these algorithms to find an estimated causal data structure.
After that, we compute the discrepancy between the estimated and real structures,
including the two types of errors.
Finally, we aggregate the discrepancy statistics of all 1,000 datasets for each of the structures in Figures \ref{fig:classification} to \ref{fig:subexogeneity}.

In Table \ref{tbl:precision_estimation}, the column of Exact Identification lists the cases when the estimated structure is the same as the figure; no causal relation is missed or added to the estimated structure. 
However, the estimated structure could miss one or more causal relations. 
These are listed in the column of Number of Omissions.
The estimation could add one or more additional causal relations which do not exist in the figure; they are in the column of Number of Additions.
The last column lists the percentage of correctly identified causal relations.
For example, there are ten causal relations in Figure \ref{fig:classification} and 91\% of them are correctly identified in the estimation.

\begin{table}[!h]
\caption{Accuracy of Identification in 1,000 Cause-Effect Models}
\label{tbl:precision_estimation}
\centering
\begin{tabular}{r||c||c|c|c|c || c|c|c|c|| c}\hline
Causal                      &Exact         &\multicolumn{4}{c||}{Number of Omissions}&\multicolumn{4}{c||}{Number of Additions}&Cause-Effect\\
\cline{3-10}
Structure                   &Identification&\hspace*{2.2mm}1\hspace*{2.2mm}&\hspace*{2.2mm}2\hspace*{2.2mm}&\hspace*{2.2mm}3\hspace*{2.2mm}&$\ge 4$&\hspace*{2.2mm}1\hspace*{2.2mm}&\hspace*{2.2mm}2\hspace*{2.2mm}&\hspace*{2.2mm}3\hspace*{2.2mm}&$\ge 4$&Identification\\ \hline
Fig \ref{fig:classification}&863  &52  &39  &29  &17 &21  &56  &23  &29  &91.4\%\\
Fig \ref{fig:hierarchy}     &802  &22  &41  &47  &88 &18  &63  &52  &48  &86.1\%\\
Fig \ref{fig:circular}      &817  &1   &12  &67  &103&65  &43  &17  &44  &84.5\%\\
Fig \ref{fig:periodic}      &915  &9   &15  &19  &42 &7   &26  &12  &32  &95.3\%\\
Fig \ref{fig:subexogeneity} &771  &7   &35  &62  &125&23  &36  &33  &121 &82.2\%\\\hline
Average                     &833.6&18.2&28.4&44.8&75 &26.8&44.8&27.4&54.8&87.9\%\\\hline
\end{tabular}
\end{table}

Table \ref{tbl:precision_estimation} summarizes the discrepancy statistics for each of the five figures in its 1,000 simulated datasets.
The average precision is around 83\% when comparing the estimated figures and the figures to be estimated. 
About 4.5\% of cases make a slight mistake that either fails to find only one true causal relation or takes a wrong causal one.
Over 13\% of cases make a big mistake that fails to identify at least four true causal relations or takes at least four wrong causal ones.

\subsection{Determinants of Climate Changes}

Geological records show that there have been some significant variations in the Earth's climate over the past hundreds of years since the industrial revolution. 
These have been caused by many factors, either natural or human. 
Understanding the contributions by the causal factors helps policy-makers make the right decisions to resolve the problem.

Determinants of climate change have been well studied in the literature (e.g., Crowley, 2000; Pasgaard and Strange, 2013; Stern and Kaufmann, 2014), but viewpoints are also diversified, partially due to the different interpretations of causality.
We, therefore, use the commonly agreed determinants in the study.
These come from three categories: the Sun, the Earth, and human activities. 
The categories are not exclusive; for example, carbon dioxide (CO2) could come from factories, humans' breaths, or volcanoes.
Some are exogenous to the effects, such as the Sun's radiations; others are endogenous to global warmings, such as the levels of CO2 and real GDP per capita (RGDP PC).
There also exist ambiguous causal relations, such as that between the ocean currents and volcanoes, for which we run the identification algorithms.
Table \ref{tbl:climate_causality} lists $15$ time series for the VARs.
Because of the sample size of the data, many factors are not included, such as atmospheric aerosols.

We use data for the years between 1900 to 2019.
The data come from GitHub.com, NASA (National Aeronautics and Space Administration), the USGS (United States Geological Survey), and the World Bank.
Different time series have their own cyclic and trending patterns.
The duration of the sunspot cycle, for example, is around eleven years, while El Ni$\mathrm{\tilde n}$os and La Ni$\mathrm{\tilde n}$as generally occur about every two to seven years.
Also, extrapolation of the estimated relation to an infinite horizon could be subject to many disastrous venerabilities.
For example, the Earth's cool and warm periods cycle roughly every 100,000 years, caused by changes in Earth's orbit around the Sun. 
For this reason, we use (\ref{eq:omega_ijh}) with $h=120$ as the elements of $\Omega$, $\Omega^{120}$ for the global causality distribution, and $\mu^{(120)}$ in Corollary \ref{cr:limit_class_causality} as the class causality.
With this finite horizon, $\pi_i$ may not be zero for an endogenous series, which may have a nonzero local causality on another endogenous series.
For the VARs, we choose $20$ as the lag length, which covers the cycles of the Sun, ocean currents, and economic business cycles.
However, the sample size $120$ is not enough to afford $15$ time series for all 20 lags.
Therefore, based on lag exclusion tests, the VARs finally pick lags 1, 2, 5, 7, 12, and 20.

\begin{table}[ht]
\caption{The causality distribution to global warming (1900$-$2019)}
\label{tbl:climate_causality}
\centering
\begin{tabular}{r|c|c || r|c|c}\hline
Determinant                          &Causality&.95 Confidence&Determinant&Causality&.95 Confidence\\\hline
Solar irradiance                     &19.3     &[17.5,20.8]   &RGDP PC    &2.1      &[1.8,2.4]\\
Milankovitch cycle                   &13.6     &[12.7,14.4]   &Population &4.4      &[3.9,4.9]\\
El Ni$\mathrm{\tilde n}$o oscillation&4.9      &[4.2,5.5]     &CO2        &16.7     &[14.5,18.3]\\
Volcanic eruption                    &6.7      &[5.7,7.5]     &Ozone O3   &7.1      &[6.4,7.9]\\
Vegetation cover                     &7.2      &[6.3,8.1]     &Methane    &4.7      &[4.3,5.0]\\
Nitrous oxide                        &5.9      &[5.0,6.7]     &Water vapor&5.6      &[4.8,6.5]\\\hline
\end{tabular}
\end{table}

Table \ref{tbl:climate_causality} summarizes the local causality distribution for global warming.
Causal identification algorithms show that Solar irradiances, Earth's Milankovitch cycles and volcanic eruptions are three independent exogenous variables,
and the other series are exogenous.
The three exogeneity classes account for 39.6\% of global warming.
This percentage mostly comes from indirect effects, such as the greenhouse.
Combined with the El Ni$\mathrm{\tilde n}$o and part of CO2 and Methane, the natural factors make up more than half of the responsibility for global warming.
Among the greenhouse gases that mitigate the infrared radiation from the Earth, carbon dioxide emission has about 50\% share.
Lastly, changes in population and vegetation cover also contribute 12\%.

The endogenous factors interact with each other in loops. 
For example, increased atmospheric temperature evaporates more water which makes the greenhouse thicker.
Conversely, a thicker greenhouse increases the temperature, which evaporates more water.
In this loop, each factor gets its causality distribution, which answers the question of Thurman and Fisher (1988).
A policy-maker could capitalize on the multiplier effect in the loop.
If he or she decreases the level of CO2 by 10\%, for example, the direct effect is a deduction of global warming by 1.67\%.
Nevertheless, the spillover effects make the deduction even further.

\section{Conclusions}\label{sect:conclusion}
 
This paper provides a game-theoretic framework to study a fundamental issue in statistical learning and econometrics.
When observing the bilateral influences over a group of time series, we set up a directional network by a stochastic matrix.
However, we could not simply apply a causal interpretation to the matrix as transitivity is invalid, noisy causal relations are not filtered, and transient causality is spuriously embedded.
By introducing a counterbalance to the matrix, the solution mitigates the noise, makes transitivity to a class property, and zeros out the impact from causal-transient data.
The set of time series is decomposed into a few disjoint classes, either exogenous or endogenous to others.
Within each exogeneity class, a counterbalance could also be set up, and its solution quantifies the causality for the movement of the entire class.
We also set up a similar counterbalance for the endogeneity class.
This balance is subject to external influence from exogeneity classes.
The solution explains the reason and responsibility of the movement of any particular endogenous variable.
However, the equilibrium is set at the long-run level, not at the short-run dynamics, because a real causal factor could affect over a long time and involve many intermediaries. 
We convert the causal identification into a hypothesis testing statistical decision problem by estimating the expected effect and uncertainty.

A few advantages could explain the method's superior performance over other methods. 
First, by admitting the indirect influences from the third parties, the limit bilateral impacts are consistent and resolve internal conflicts.
Secondly, the formula and estimation procedures are simple and easy to implement.
Also, as illustrated in the simulation studies, the performance is not perfect but good enough for a small-scale analysis.
Next, by normalizing the sum of rows in $\Omega$, we lay all the noise of $\Omega$ in a Procrustean bed, allowing only zero-sum for the noise in each row.
The procrustean ruled noise is further offset in the sum of infinitely linked chains of causality in which both positive and negative noise is present.
Finally, the causal inference called the causality distribution for an effect variable is a byproduct of the causal discovery. 
The distribution fairly allocates the responsibility of the effect to the source factors. 
In practice, for example, one could apply it to taxation based on the causality distribution of the pollution, climate change, income, or social inequality.

One could extend the solution from different angles, which are worth further development.
First, linear functions may be too simple in a complicated data analysis, though linear methods could be tried first before moving on to nonlinear alternatives.
However, formulating a stochastic matrix like $\Omega$ is still possible; the matrix should summarize the directional influence between any two time series, even in a nonlinear setting. 
The $h$-step forecast error variance decomposition (\ref{eq:omega_ijh}) is a slightly different alternate, especially when $h$ is the lag length of the VAR or the sample size of the data.
Secondly, the algorithms have space to improve the accuracy. 
The 83\% precision for the basic causal structures could imply a much-discounted accuracy for a complex situation with many variables.
The causal relation may be tested for any pair of time series, ignoring the computational cost and the class property.
Thirdly, the transition matrix $\Omega$ captures the variability of the time series vector.
Positive and negative changes, however, could have asymmetric roles in the data movement process. 
Next, not all data structures work well in these algorithms, for example, for extremely weak exogeneity. 
Of course, the results also depend on the lag length selection and other specifications of the VARs. 
Finally, in the algorithms, we capitalize on the bootstrapping for the standard errors of the $\pi$.
The data re-sampling technique not only involves more computational cost but also brings new sampling errors.
A direct way is to calculate the standard errors using (\ref{eq:VAR_limit_variance})-(\ref{eq:VAR_limit_variance_yj}) or (\ref{eq:limit_variance_SVAR})-(\ref{eq:limit_variance_SVAR_yj}) to avoid costs and new errors.

Most economists at one time or another have probably found themselves in an unconscious position of mixing up correlation with causation, though not because they consciously know the differences. 
At such times it may be convenient to have an illustration at hand to show the limitations of regression, rather than a list of causality definitions or inferences. 
What, then, is causality? 
The answer, it appears, is that any argument carried out with sufficient precision is mathematical logic. 
VAR or regression modeling is not necessary, but consistency within any causal chain is. 
Seeing through the appearance to perceive the essential determinants is another necessity. 
Lastly, the argument matches people's common sense of causality, such as those pictured in Figures \ref{fig:classification} to \ref{fig:subexogeneity}. 
For these considerations, the foregoing causality distributions may serve as two practical solutions.

%

%

\newpage
\section*{Appendix: The Proofs}

\subsection*{A1. Proof of Theorem \ref{thm:limit_variance_decomposition}}\label{prf:thm:limit_variance_decomposition}
\noindent
By the Lyapunov Theorem, the solution to (\ref{eq:Lyapunov}) exists if and only if the time series $y_{.t}$ is asymptotically stable.
So the limit of the denominator in (\ref{eq:omega_ijh}) exists as $h\to \infty$, and the limit is an entry of $X$.
The numerator also has a limit because it is increasing with $h$ and bounded by the denominator.

The ratio is unconditional because both the numerator and denominator of (\ref{eq:omega_ijh}) are already unconditional.
$\diamondsuit$

\subsection*{A2. Proof of Theorem \ref{thm:pattern_exgo_endo}} \label{prf:thm:pattern_exgo_endo}
\noindent
First, say, the pattern of (\ref{eq:pattern_Endo_Exog}) has the sizes of 
$$
\left (
\begin{array}{cc}
(n-m)\times (n-m)& \\
m\times (n-m)    &\hspace*{4mm}m\times m
\end{array}
\right ).
$$
That is, there are $n-m$ variables in the exogeneity class and $m$ in the endogeneity class.
Clearly, the product of any two matrices with the pattern also has the same pattern.

Secondly, the moving averaging representation of $y_{t+h}$ is
$$
y_{t+h} = \sum\limits_{s=0}^\infty \phi_s u_{t+h-s}
$$
where the matrix $\phi_s$ can be recursively deduced from 
\begin{equation}\label{eq:calc_phi}
\phi_s = A_1 \phi_{s-1} + A_2 \phi_{s-2} + \cdots + A_p \phi_{s-p}
\end{equation}
with the starting values $\phi_0 = I_n$ and $\phi_s=0_n 0_n'$ for any $s<0$.
By the mathematical induction, clearly, $\phi_s$ has the same pattern for any $s=0,1,\cdots$.

Next, $\Sigma$'s Cholesky component $L_j$ for any endogenous $y_j$ has zeros in its first $n-m$ elements as $y_j$ is after these $n-m$ exogenous variables in the VAR. 
With the pattern of $\phi_s$, $\phi_s L_j$ also has zeros in its first $n-m$ elements because $\phi_s L_j$ is a linear combination of the last $m$ columns of $\phi_s$.
Therefore, $(\phi_s L_j) (\phi_s L_j)'$ has a zero block in its upper right corner.
Finally, $\Omega$ has the same pattern because $y_j$'s contribution to $\mathrm{cov}(y_{t+h})$ is
$$
\sum\limits_{s=0}^\infty \ (\phi_s L_j) (\phi_s L_j)'.
$$
$\diamondsuit$

\subsection*{A3. Proof of Theorem \ref{thm:invariant_linearity}} \label{prf:thm:invariant_linearity}
\noindent
Without loss of generality, $y_1$ has a linear transformation
$$ 
\tilde y_1 = \alpha + \beta y_1
$$
and the other time series remain the same.
For the new VAR of $\{\tilde y_1, y_2, \cdots, y_n\}$,
let $\tilde A_i$ be the new coefficient matrix and $\tilde \Sigma$ the variance-covariance matrix of the new residuals $\tilde u_t$.

The relations between $A_i$ and $\tilde A_i$ and between $\Sigma$ and $\tilde \Sigma$ are as follows:
$\tilde A_i = F A_i F^{-1}$ and $\tilde \Sigma = F \Sigma F$ where 
$
F =
\left [ \begin{array}{cc} \beta & \\ & I_{n-1} \end{array} \right]
$
and $F' = F$.
In $\tilde y_1$'s equation of the new VAR, all coefficients are multiplied by $\beta$ from those of the original VAR in (\ref{eq:GUVAR}).
When the lag value of $\tilde y_1$ acts as regressors, all its coefficients are divided by $\beta$, compared to those of the original VAR.
As the residual $u_{1t}$ is scaled by $\beta$ in the new VAR, so it variance is scaled by $\beta^2$ and its covariances with other residuals are scaled by $\beta$.

We also write the moving averaging representation for $\tilde y_{.t}$ as $\tilde y_{.t} = \sum\limits_{s=0}^\infty \tilde \phi_s \tilde u_{t-s}$ where
$\tilde \phi_s = F \phi_s F^{-1}$ because of (\ref{eq:calc_phi}) and $\tilde A_i = F A_i F^{-1}$.
Also, $\tilde u_{t-s} = F u_{t-s}$ and the new Cholesky decomposition is $\tilde L_j = F L_j$ because $\tilde \Sigma = F \Sigma F = F L L' F = (FL)(FL)'$.

Now the covariance of $\tilde y_{.t}$ is
\begin{equation}\label{eq:var_tilde}
\mathrm{cov}(\tilde y_{.t}) = \sum\limits_{s=0}^\infty \tilde \phi_s \tilde \Sigma \tilde \phi_s' 
= 
\sum\limits_{s=0}^\infty F \phi_s F^{-1} F \Sigma F F^{-1} \phi_s' F
=
F \mathrm{cov}(y_{.t}) F
\end{equation}
and its contribution by $\tilde y_j$ is
\begin{equation}\label{eq:var_tilde_j}
\sum\limits_{s=0}^\infty \tilde \phi_s \tilde L_j \tilde L_j' \tilde \phi_s' 
= 
\sum\limits_{s=0}^\infty F \phi_s F^{-1} F L_j L_j' F F^{-1} \phi_s' F
=
F \left [\sum\limits_{s=0}^\infty \phi_s L_j L_j' \phi_s' \right ] F.
\end{equation}

In the diagonals of (\ref{eq:var_tilde_j}) and (\ref{eq:var_tilde}), the only changes from the original VAR in (\ref{eq:GUVAR}) are the first elements, both multiplied by $\beta^2$.
Therefore, the ratios between the two diagonals remain the same as $\omega_{ij}$ for all $i=1,2,\cdots,n$.  
$\diamondsuit$

\subsection*{A4. Proof of Theorem \ref{thm:SVAR_variance_decomposition}}\label{prf:thm:SVAR_variance_decomposition}
\noindent
Using (\ref{eq:SVAR2VAR}), we write
$$
Y_t = \left [
\begin{array}{cc}
(I_n-A_0)^{-1}& \\
              & I_{n(p-1)}	
\end{array}
\right ]
\left [
V+A Y_{t-1}+U_{t}
\right ].
$$
So, for the coefficient matrix of $Y_{t-1}$, $|(I_n-A_0)^{-1}A|\le |(I_n-A_0)^{-1}| |A|<1$.
$\diamondsuit$

\subsection*{A5. Proof of Theorem \ref{thm:dpii_dOmegaji}}\label{prf:thm:dpii_dOmegaji}
\noindent 
When we make a small perturbation $\Delta \Omega$ to $\Omega$, the new causality distribution $\pi + \Delta \pi$ satisfies the counterbalance equation of
\begin{equation}\label{eq:delta_authority_distribution}
\pi + \Delta \pi = (\pi + \Delta \pi) [\Omega + \Delta \Omega]
\end{equation}
subject to $\Delta \Omega 1_n = 0_n$ and $\Delta \pi 1_n = 0$.
After subtracting $\pi = \pi \Omega$ from (\ref{eq:delta_authority_distribution}), we get
$$
\Delta \pi [I_n - \Omega -\Delta \Omega] = \pi \Delta \Omega
$$
and its first-order approximation
\begin{equation}\label{eq:linear_equations}
\Delta \pi [I_n - \Omega] \approx \pi \Delta \Omega.
\end{equation}

Let us increase $\omega_{ji}$ by $\Delta \omega_{ji}$ and calculate the effect of the change on $\pi$. 
The elements of $\Delta \Omega$ are all zeros except the $j$th row.
After the increase of $\omega_{ji}$, the sum of the $j$th row is $1+\omega_{ji}$.
To maintain the unit sum of the $j$th row, we divide the $j$th row by $1+\omega_{ji}$.
Compared to their original values, thus, the other elements in the row decrease roughly proportionally. 
Thus, the linear approximation of the $j$th row of $\Delta \Omega$ is
$$
\Delta \omega_{ji} \left (\frac{-\omega_{j1}}{1-\omega_{ji}}, \cdots, \frac{-\omega_{j,i-1}}{1-\omega_{ji}}, 1, \frac{-\omega_{j,i+1}}{1-\omega_{ji}}, \cdots, \frac{-\omega_{jn}}{1-\omega_{ji}} \right )
\stackrel{\mathrm{def}}{=\joinrel=}  
\Delta \omega_{ji} \beta
$$
where $\beta$ is the row vector delimited by the parenthesis.
Thus, $\Delta \Omega = \Delta \omega_{ji} e_j \beta + o(\Delta \omega_{ji})$ where $e_j$ is the $j$th column of $I_n$.
Dividing (\ref{eq:linear_equations}) by $\Delta \omega_{ji}$ and letting $\Delta \omega_{ji} \to 0$, we have the derivative of $\pi$ with respect to $\omega_{ji}$:
$$
\frac{\mathrm{d} \pi}{\mathrm{d} \omega_{ji}} \left [ I_n - \Omega \right ] 
=
\pi \frac{\mathrm{d} \Omega}{\mathrm{d} \omega_{ji}} 
=
\pi e_j \beta 
=
\pi_j \beta,
$$
or, in the usual column representation of unknowns, 
\begin{equation}\label{eq:partial_pi}
\left [ I_n - \Omega' \right ] \frac{\mathrm{d} \pi'}{\mathrm{d} \omega_{ji}} 
=
\pi_j \beta'.
\end{equation}

To solve $\frac{\mathrm{d} \pi'}{\mathrm{d} \omega_{ji}}$ from (\ref{eq:partial_pi}) and the identity $\frac{\mathrm{d} \pi}{\mathrm{d} \omega_{ji}} 1_n = 0$, we write the augmented matrix for $\frac{\mathrm{d}\pi'}{\mathrm{d} \omega_{ji}}$ as
\begin{equation}\label{eq:augumented_matrix}
\left [
\begin{array}{c|c}
1_n'                                 &0\\
\hspace*{4mm}I_n-\Omega'\hspace*{4mm}&\pi_j\beta' \\
\end{array}
\right ]
\end{equation}
In the $(n+1)\times (n+1)$ matrix, the sum of the last $n$ rows is a zero row vector.
We delete the $i+1$st row to drop the collinearity. 
We also move the $i$th column to the first without changing the order of other columns; this movement reorders the $i$th variable in $\mathrm{d}\pi'$.
The result is the augmented matrix for $\frac{\mathrm{d} (\pi_i, \pi_{-i})'}{\mathrm{d} \omega_{ji}}$,
\begin{equation}\label{eq:augumented_matrix}
\left [
\begin{array}{cc|c}
\hspace*{6mm}1\hspace*{6mm}&1_{n-1}'   &0 \\
-\alpha_{ii}               &I_{n-1}-Z_i&\frac{-\pi_j}{1-\omega_{ji}} \alpha_{ji} \\
\end{array}
\right ]
\end{equation}

We multiply 
$
\left [
\begin{array}{cc} 
1\hspace*{4mm}&-1_{n-1}' (I_{n-1}-Z_i)^{-1} \\
0\hspace*{4mm}&I_{n-1}
\end{array}
\right ]
$
to the left side of (\ref{eq:augumented_matrix}) to get the following new augmented matrix
\begin{equation}\label{eq:augumented_matrix_new_ii}
\left [
\begin{array}{cc|c}
1+1_{n-1}'(I_{n-1}-Z_i)^{-1}\alpha_{ii}&0                       &\frac{\pi_j}{1-\omega_{ji}} 1_{n-1}'(I_{n-1}-Z_i)^{-1} \alpha_{ji}\\
\hspace*{7mm}-\alpha_{ii}\hspace*{7mm} &\hspace*{4mm}I_{n-1}-Z_i&\frac{-\pi_j}{1-\omega_{ji}} \alpha_{ji} \\
\end{array}
\right ].
\end{equation}
As $(I_{n-1}-Z_i)^{-1} = I_{n-1}+Z_i+Z_i^2+Z_i^3+\cdots$, all its elements are non-negative. 
Therefore, the first equation in (\ref{eq:augumented_matrix_new_ii}) implies that
$$
\frac{\mathrm{d} \pi_i}{\mathrm{d} \omega_{ji}} 
= 
\frac{\pi_j}{1-\omega_{ji}} \
\frac{1_{n-1}' (I_{n-1}-Z_i)^{-1} \alpha_{ji}}{1+1_{n-1}' (I_{n-1}-Z_i)^{-1}\alpha_{ii}} \ge 0.
$$
Also, the second equation of (\ref{eq:augumented_matrix_new_ii}) implies that
$$
-\alpha_{ii} \frac{\mathrm{d} \pi_i}{\mathrm{d} \omega_{ji}} + [I_{n-1}-Z_i] \frac{\mathrm{d} \pi_{-i}'}{\mathrm{d} \omega_{ji}} 
= 
\frac{-\pi_j}{1-\omega_{ji}} \alpha_{ji}.
$$
Therefore,
$$
\frac{\mathrm{d} \pi_{-i}'}{\mathrm{d} \omega_{ji}} 
=
[I_{n-1}-Z_i]^{-1} \left [ \frac{\mathrm{d} \pi_i}{\mathrm{d} \omega_{ji}} \alpha_{ii} - \frac{\pi_j}{1-\omega_{ji}} \alpha_{ji}\right ].
$$
$\diamondsuit$

\subsection*{A6. Proof of Theorem \ref{thm:zero_limit_Te}}\label{prf:thm:zero_limit_Te}
\noindent
By the Perron-Frobenius theorem, the largest eigenvalue (in magnitude) of $T_e$ is a real positive number, say $\rho$, and its eigenvector $v$ is a positive vector.
So, $T_e v = \rho v$.
Also, $0\le \min\limits_{i \in T} \sum\limits_{j \in T} T_e(i,j) \le \rho\le \max\limits_{i \in T} \sum\limits_{j \in T} T_e(i,j) \le 1$.

We claim that $\rho<1$. Otherwise, if $\rho=1$, we define the set of elements which have the largest value in $v$:
$$
\Psi_v = \{ i\in T | v_i = \max\limits_{j \in T} v_j\}.
$$
Clearly, the set $\Psi_v$ is not empty.
For any $i\in \Psi_v$, for the equality
$$
v_i = \sum\limits_{j \in T} T_e(i,j) v_j \le \sum\limits_{j \in T} v_i T_e(i,j) \le v_i
$$
to hold in $T_e v = \rho v$, we must have $v_j=v_i$ whenever $T_e(i,j)>0$, or $j\in \Psi_v$ whenever $T_e(i,j)>0$.
Thus, $\Psi_v$ is an independent exogeneity class because $T_e(i,j)=0$ for any $i\in Z$ and $j\not \in Z$.
We get a contradiction to the causal-transient class $T$.
Therefore, $\rho<1$.

We write $T_e$ as the Jordan normal form $T_e = F J_e F^{-1}$ where $F$ is an invertible matrix.
The Jordan matrix $J_e$ has $T_e$'s eigenvalues on the diagonal and $1$s or $0$s on the superdiagonal.
Thus, $\lim\limits_{t \to \infty} J_e^t = 0_m 0_m'$.
Finally, $\lim\limits_{t \to \infty} T_e^t = \lim\limits_{t \to \infty} F J_e^t F^{-1} =  F (\lim\limits_{t \to \infty} J_e^t )F^{-1} = 0_m 0_m'$.
$\diamondsuit$

\subsection*{A7. Proof of Corollary \ref{cr:zero_pi_invertible}}\label{prf:cr:zero_pi_invertible}
\noindent
First, if there exists a vector $x \in R^m$ such that
$(I-T_e) x = 0_m$, then $x = T_e x$ and, as a consequence, $x = T_e^t x$ for all $t=1,2,\cdots$. 
Letting $t\to\infty$ gives that $x = 0_m$ by Theorem \ref{thm:zero_limit_Te}. 
Therefore, $I_m-T_e$ is invertible.

Next, if $T\not = \emptyset$ and $y_{.t}$ contains no exogeneity class,
then $N$ itself is a causal-transient coalition and $\Omega^t$ converges to a zero matrix as $t\to \infty$.
So $\Omega^t 1_n \to 0_n$. 
However, this contradicts to $\Omega^t 1_n = 1_n$ for any $t=1,2,\cdots$. 

Finally, by $\pi = \pi \Omega$ and the causal-transiency of $T$,
$\pi_i = \sum\limits_{j\in T} \pi_j T_e^t  (j,i)$ for all $t \ge 0$.
Letting $t\to\infty$ results in $\pi_i = 0$, by Theorem \ref{thm:zero_limit_Te}.
$\diamondsuit$

\subsection*{A8. Proof of Theorem \ref{thm:pi_with_quota}}\label{prf:thm:pi_with_quota}
In 
$
\Omega = \left (
\begin{array}{c c c c}
\Omega_1&      &        &   \\
    	&\ddots&        &   \\
	    &      &\Omega_k& 
\end{array}
\right ),
$
each $\Omega_i$ is an aperiodic irreducible stochastic matrix.
So, $\lim\limits_{t \to \infty} \Omega_i^t$ exists.
We set $\pi^{c_i}$ by $\lim\limits_{t \to \infty} \Omega_i^t = \frac{1}{q_i} 1_{_{|c_i|}} \pi^{c_i}$ where $1_{_{|c_i|}}$ has the size of $C_i$. 
Then the row vector $(\pi^{c_1}, \pi^{c_2}, \cdots, \pi^{c_k})$ solves the counterbalance equation and also satisfies the quota condition.
$\diamondsuit$

\subsection*{A9. Proof of Theorem \ref{thm:class_causality2i}}\label{prf:thm:class_causality2i}
\noindent
In the iteration (\ref{eq:iteraton_counterbalance_external}), $\mu^{(t)}$ is an increasing vector as $t\to \infty$ because
$$
\mu^{(t+1)}-\mu^{(t)} 
= 
T_e (\mu^{(t)}-\mu^{(t-1)})
=
\cdots
= 
T_e^t (\mu^{(1)}-\mu^{(0)})
=
T_e \mu^{(1)}
\ge 
0_m.
$$
Next, clearly, $\mu^{(0)} \le 1_m$. 
If we assume that $\mu^{(t)}\le 1_m$ for some $t>0$, then by (\ref{eq:iteraton_counterbalance_external}),
$$
\mu^{(t+1)}
\le T_e 1_m + \left  ( 
	\begin{array}{c} 
		\sum\limits_{j \in C} \omega_{n-m+1,j} \\
		\vdots \\
		\sum\limits_{j \in C} \omega_{nj} \\
	\end{array}
	\right ) 
=
\left  ( 
\begin{array}{c} 
\sum\limits_{j \in C\cup T} \omega_{n-m+1,j} \\
\vdots \\
\sum\limits_{j \in C\cup T} \omega_{nj} \\
\end{array}
\right ) 
\le 1_m.
$$
By the mathematical induction, $\mu^{(t)}\le 1_m$ for all $t$.
Therefore, $\lim\limits_{t\to\infty} \mu^{(t)}$ exists. 
The limit satisfies (\ref{eq:counterbalance_external}) and is between $0_m$ and $1_m$. 

If (\ref{eq:counterbalance_external}) has two different solutions $\mu^c$ and $\delta^c$,
without loss of generality, say, $\tau = \max\limits_{j\in T} \{ \mu^c_j - \delta^c_j\} >0$.
Let
$$
Z = \left \{ i\in T | \mu^c_i -\delta_i^c = \tau \right \}.
$$
Then for any $i\in Z$, for the equality 
$$
\tau
=
\mu^c_i -\delta_i^c 
= \sum\limits_{j\in T} \omega_{ij} \left [\mu^c_j -\delta_j^c \right ] 
\le \sum\limits_{j\in T} \omega_{ij} \tau 
\le \tau
$$
to hold if and only if $\mu^c_j -\delta_j^c = \tau$ whenever $\omega_{ij}>0$.
That is, $j\in Z$ whenever  $\omega_{ij}>0$ and $\omega_{ik}=0$ whenever $k\not \in Z$.
Therefore, $Z$ is an exogeneity class and we get an contradiction to the causal-transient class $T$.  
This contradiction indicates that $\tau$ can not be positive and the solution to (\ref{eq:counterbalance_external}) must be unique.
$\diamondsuit$

\subsection*{A10. Proof of Corollary \ref{cr:limit_class_causality}}\label{prf:cr:limit_class_causality}
\noindent
(By induction) Clearly, it is true for $t=1$. 
If we assume that it's true for some $t\ge 1$, then 
$$
\begin{array}{rcl}
\sum\limits_{j \in C} \Omega^{t+1} (i,j) 
&=&
\sum\limits_{j \in C} \left [ \sum\limits_{k\in T} \omega_{ik} \Omega^t (k,j) + \sum\limits_{k\in C} \omega_{ik} \Omega^t (k,j)\right ] \\
&=&
\sum\limits_{k\in T} \omega_{ik}\sum\limits_{j \in C} \Omega^t(k,j) + \sum\limits_{k\in C} \omega_{ik} \sum\limits_{j \in C} \Omega^t (k,j) \\
&=&
\sum\limits_{k\in T} \omega_{ik} \mu_k^{(t)} + \sum\limits_{k\in C} \omega_{ik} 1 \\
&=&
\mu_i^{(t+1)}.
\end{array}
$$
By the principle of induction, the statement is true for all $t\ge 1$.

Next, the proof of Theorem \ref{thm:class_causality2i} already shows that
$
\lim\limits_{t\to\infty} \mu^{(t)} = \left ( \begin{array}{c} 
	\mu_{n-m+1}^c \\ \vdots \\ \mu_n^c \\ \end{array} \right ).
$ 

Finally, by (\ref{eq:counterbalance_external}),
$\left ( 
\begin{array}{c}
	\mu_{n-m+1}^c \\ \vdots \\ \mu_n^c \\
\end{array} 
\right )
= T_e \left  (
\begin{array}{c}
	\mu_{n-m+1}^c \\ \vdots \\ \mu_n^c \\
\end{array} 
\right ) 
+ \left  ( 
\begin{array}{c} 
	\sum\limits_{j \in C} \omega_{n-m+1,j} \\
	\vdots \\ 
	\sum\limits_{j \in C} \omega_{n,j} \\
\end{array}
\right ).$
Therefore, we have the expression of (\ref{eq:calculate_mu_c_i}) by Corollary \ref{cr:zero_pi_invertible}.
The solution is nonnegative because $T_e$ and $(I_m-T_e)^{-1}=I_m+T_e+T_e^2+T_e^3+\cdots$ are nonnegative matrices.
$\diamondsuit$

\subsection*{A11. Proof of Theorem \ref{thm:total_absorbtion}}\label{prf:thm:total_absorbtion}
\noindent
By (\ref{eq:calculate_mu_c_i}),
$$
\begin{array}{rcl}
\sum\limits_{s=1}^k \left (
\begin{array}{c}
	\mu_{n-m+1}^{c_s} \\
	\vdots \\
	\mu_n^{c_s} \\	
\end{array} \right )
&=& 
\left [I_m-T_e \right ]^{-1}  \left  ( 
\begin{array}{c} 
\sum\limits_{s=1}^k \sum\limits_{j \in C_s} \omega_{n-m+1,j} \\
\vdots \\ 
\sum\limits_{s=1}^k \sum\limits_{j \in C_s} \omega_{nj} \\
\end{array}\right ) \\
&=& 
\left [ I_m + T_e+ T_e^2 + T_e^3 + \cdots \right ] \left ( 1_m - T_e 1_m \right ) \\
&=&
\left [ I_m + T_e+ T_e^2 + \cdots \right ]1_m - \left [ I_m + T_e+ T_e^2 + \cdots \right ] T_e 1_m  \\
&=&
I_m 1_m \\
&=&
1_m.
\end{array}
$$ 
$\diamondsuit$

\subsection*{A12. Proof of Corollary \ref{cr:limit_Omega_T}}\label{prf:cr:limit_Omega_T}
\noindent
By Theorem \ref{thm:total_absorbtion} and Corollary \ref{cr:limit_class_causality}, $\lim\limits_{t \to \infty} \sum\limits_{j \in T} \Omega^{t}(i,j) = 0$ for any $i\in T$.
$\diamondsuit$

\subsection*{A13. Proof of Theorem \ref{thm:responsibility_index}}\label{prf:thm:responsibility_index}
\noindent
For any $i\in T$ and $j\in C$, by Corollary \ref{cr:limit_Omega_T} and $\Omega^{t+1} = \Omega^t \Omega$,
$$
\begin{array}{rcl}
\lim\limits_{t\to\infty} \Omega^{t+1} (i,j)
&=&
\lim\limits_{t\to\infty} \sum\limits_{k\in T} \Omega^t (i,k) \omega_{kj} + 
\lim\limits_{t\to\infty} \sum\limits_{k\in C} \Omega^t (i,k) \omega_{kj} \\
&=&
\sum\limits_{k\in T} \omega_{kj} \lim\limits_{t\to\infty} \Omega^t(i,k) +
\sum\limits_{k\in C} \lim\limits_{t\to\infty} \Omega^t(i,k) \omega_{kj} \\
&=&
\sum\limits_{k\in T} 0 \omega_{kj} + \sum\limits_{k\in C} \lim\limits_{t\to\infty} \Omega^t (i,k) \omega_{kj} \\
&=& 
\sum\limits_{k\in C} \left (\lim\limits_{t\to\infty} \Omega^t (i,k) \right ) \omega_{kj}
\end{array}
$$
which is a counterbalance equation in $C$.
Combining with the aggregate limit in Corollary \ref{cr:limit_class_causality}, we solve the counterbalance equations and the quota restriction of
$$
\left \{
\begin{array}{l}
\lim\limits_{t\to\infty} \Omega^t (i,j) 
= 
\sum\limits_{k\in C} \left (\lim\limits_{t\to\infty} \Omega^t (i,k) \right ) \omega_{kj}, \\
\sum\limits_{j\in C} \lim\limits_{t\to\infty} \Omega^t(i,j) = \mu_i^c, 
\end{array} 
\right .
$$
to get $\lim\limits_{t \to \infty} \Omega^t (i,j) = \pi_j^c \mu_i^c$ by Theorem \ref{thm:pi_with_quota}.
$\diamondsuit$

\subsection*{A14. Proof of Corollary \ref{cr:sum_casuality_index}}\label{prf:cr:sum_casuality_index}
\noindent
By Theorem \ref{thm:total_absorbtion},
$$
\sum\limits_{s=1}^k \sum\limits_{j \in C_s} \pi_j^{c_s} \mu_i^{c_s} 
= 
\sum\limits_{s=1}^k 1 \mu_i^{c_s} 
=
1. 
$$
$\diamondsuit$

\end {document}